\documentclass[11pt]{article}

\usepackage[margin=1in]{geometry}
\usepackage{amsthm,amsmath,amsfonts,amssymb}
\usepackage[numbers,sort&compress]{natbib}
\usepackage[dvipsnames]{xcolor}
\usepackage{etoolbox}
\usepackage{graphicx}
\usepackage{booktabs,tabularx}
\usepackage{authblk}
\usepackage[colorlinks,citecolor=blue,urlcolor=blue]{hyperref}
\usepackage[nameinlink,noabbrev]{cleveref}
\usepackage{fancyhdr}

\usepackage{times}
\usepackage{tikz}
\usepackage[english]{babel}
\usepackage{amssymb}
\usepackage{verbatim}
\usepackage{lineno}
\usepackage{calligra}

\usepackage{amsthm}
\usepackage{amsmath}
\usepackage{amssymb}
\usepackage{graphicx}
\usepackage{verbatim}
\usepackage{cleveref}  
\usepackage{enumitem}
\usepackage{bbm}
\usepackage{csquotes}

\numberwithin{equation}{section}

\DeclareMathAlphabet{\mathbcal}{U}{dutchcal}{m}{n}

\newcommand{\Pcal}{\mathcal{P}}

\newcommand{\beq}{\begin{eqnarray*}}
\newcommand{\eeq}{\end{eqnarray*}}
\newcommand{\beqn}{\begin{eqnarray}}
\newcommand{\eeqn}{\end{eqnarray}}
\newcommand{\ben}{\begin{enumerate}}
\newcommand{\een}{\end{enumerate}}
\newcommand{\bit}{\begin{itemize}}
\newcommand{\eit}{\end{itemize}}
\newcommand{\hide}[1]{}

\newcommand{\eps}{\varepsilon}
\newcommand{\vertiii}[1]{{\left\vert\kern-0.25ex\left\vert\kern-0.25ex\left\vert #1 
    \right\vert\kern-0.25ex\right\vert\kern-0.25ex\right\vert}}

\renewcommand{\epsilon}{\eps}
\newtheorem{remark}{Remark}[section]
\newtheorem{proposition}{Proposition}[section]
\newtheorem{corollary}{Corollary}[section]

\DeclareMathOperator{\KL}{KL}
\theoremstyle{plain}

\newtheorem{theorem}{Theorem}[section]
\newtheorem{lemma}[theorem]{Lemma}
\theoremstyle{definition}
\newtheorem{definition}[theorem]{Definition}

\begin{document}

\title{Self-normalized Bernstein Inequalities for Regenerative Processes and Applications to Learning Without Independence}

\author[1]{Imon Banerjee\footnote{The authors acknowledge they used AI when preparing this manuscript. AI was not used in the mathematical research. However, AI was used to double-check the results, and to help with polishing the writing of the paper.}}
\author[2]{Pierre Alquier$^*$\footnote{This work was funded by CY Initiative, a programme supported the French government grant:
“Investissements d’avenir” \#France2030.}}
\affil[1]{University of Chicago, Booth School of Business}

\affil[2]{ESSEC Business School, Asia-Pacific campus, Singapore}

\date{}
\maketitle

\begin{abstract}
We prove a self-normalized Bernstein inequality for bounded functions of a regenerative process with independent, identically distributed marked cycles of finite mean length. The bound uses observed cycle lengths and retains explicit terms for truncation and the cycle ending after the observation horizon. Selecting a tuning parameter from a finite deterministic collection gives a version of Bernstein inequality where all terms are empirical but one, related to the tails of the regeneration times. We then study this term in many regimes: it turns out that it is of second-order in many cases of interest. This is important for applications in machine learning: we prove generalization error bounds for classification with non-i.i.d., regenerative data. These bounds can be used to define learning objectives: by minimizing the first-order term only in the bound, we define agnostic classifiers that still satisfy oracle inequalities. We also show that the bound on the average deviation provided by our results is minimax-optimal in the sample size and in the expected regeneration time. Finally, we combine the Bernstein inequality with a regenerative strong approximation to obtain a Gaussian approximation using an observable block variance estimator.
\end{abstract}

\medskip


\section{Introduction}
\label{sec:introduction}

Bernstein inequalities are fundamental tools in statistics and machine learning. By controlling deviations of empirical averages from their expectations through a variance term, they support confidence intervals, model selection and guarantees on the performance of learning algorithms \citep{bou2013,bach2024learning}. To use such guarantees in practice, the quantities appearing in the bound must be available from the observations or specified in advance. A bound involving unknown features of the underlying distribution cannot be evaluated directly and may also be difficult to use as a learning criterion. This motivates \emph{empirical}, or \emph{agnostic}, bounds, which can be evaluated without knowing the data-generating distribution. For independent and identically distributed (i.i.d.) bounded observations, empirical Bernstein inequalities replace the population variance with an observable estimate and an explicit correction \citep{maurer2009empirical}. These ideas also extend to PAC--Bayes guarantees for learning \citep{tolstikhin_seldin_2013}.

For time series, temporal dependence creates an additional obstacle. Classical extensions of Bernstein inequalities typically involve mixing coefficients or, for Markov chains, a spectral or pseudo-spectral gap \citep{samson2000concentration,paulin2015}. These quantities describe the strength of dependence but are generally unknown, so the resulting bounds cannot be evaluated directly from a single observed trajectory. Although they provide useful rates, their dependence on unknown parameters limits their direct application to statistical inference and the design of learning criteria. Estimating these parameters is one possible response. Here, we develop bounds that use observed regeneration times to account for temporal dependence. In this case, the only non-empirical term is in many regimes a second-order term. In machine learning, we show that ignoring these terms lead to agnostic classification rules who satisfy strong oracle inequalities.

We study bounded losses along a regenerative process. Regeneration divides a trajectory into cycles, allowing dependence within each cycle while providing independence between cycles. We assume that the marked cycles are i.i.d., have finite mean length and are independent of the initial delay. This framework includes regenerative Markov chains \citep{bertail_clemencon_2006} as well as processes whose observations are not Markovian.

Two features require care: (i) a bounded loss can have an unbounded cycle reward, and (ii) the number of completed cycles is random. We truncate the centred cycle rewards according to cycle length and construct an exponential supermartingale, which we stop at the first cycle ending after time \(n\). Removing that cycle yields the self-normalized Bernstein inequality in Theorem~\ref{thm:regenerative-bernstein}. Selecting its tuning parameter from a finite deterministic collection gives the empirical Bernstein bound in Corollary~\ref{cor:empirical-bernstein}. Its quadratic term is a sum of observed squared cycle lengths below the truncation level, common to all bounded losses. The bound also retains the observed contribution of longer cycles, a correction for the cycle crossing the observation horizon, and an explicit population tail expectation arising from truncation. Thus, the dependence on the cycle structure is expressed through observable quantities and a separate tail remainder, without estimating mixing coefficients or a spectral gap.

The main application we have in mind is generalization error bounds in machine learning. Here, we use our version of Bernstein inequality to prove PAC-Bayes generalization bounds for classifiers learned on non-i.i.d., regenerative data. The associated Gibbs posterior satisfies oracle inequalities both in probability and in expectation.

We then study the optimality of the scale of our bounds. Under a normalized moment condition of order greater than two on the cycle lengths, and starting at a regeneration time, we obtain a fixed-confidence estimation error of order \(\sqrt{m/n}\), where \(m\) is the mean cycle length. For integer \(m\), a matching lower bound establishes the minimax optimality of this dependence in Theorem~\ref{thm:moment-minimax}. The effective number of independent observations is therefore of order \(n/m\) in this comparison. Integrating the Bernstein inequality also gives expected deviation bounds, which we specialize under exponential and polynomial regeneration tails.

Finally, we connect the concentration bounds to asymptotic inference. Under the moment and initialization assumptions used for the minimax analysis, for a fixed loss and regenerative law with positive asymptotic variance, we use the Bernstein inequality to control empirical centering in an observable block variance estimator. Combining this control with the regenerative strong approximation of \citet{bashtova_shashkin_heavy_2020} yields a self-normalized Gaussian approximation and studentized asymptotic normality.

The rest of the paper is organised as follows. Subsection~\ref{sec:background} reviews the related literature, and Section~\ref{sec:setup} introduces the model and notation. Section~\ref{sec:bernstein} establishes the Bernstein and expected deviation bounds. While the first-order terms in these bounds are empirical, one of the second order terms is more problematic. In Section~\ref{sec:regimes}, we control this term in various regimes. Section~\ref{sec:pac} presents the PAC-Bayes and oracle inequalities. In particular, we show that ignoring the problematic terms to define a learning objective does not affect the properties of machine learning algorithms. 
Section~\ref{sec:minimax} shows that confidence intervals derived from our version of Bernstein inequality are minimax optimal.
Section~\ref{sec:regenerative-gaussian} studies the self-normalized Gaussian approximation.
Sections~\ref{sec:proofs-results} contains the proofs of all the results in the paper.

\subsection{Background and Related Literature}\label{sec:background}

Concentration inequalities for regenerative processes use independence between cycles while allowing dependence within each cycle. We refer to the monograph by \citet{bertail_clemencon_2006} for detailed literature survey of regenerative Markov chains. For a positive recurrent Markov chain with an accessible atom, \citet{BertailClemencon2010} derive tail inequalities under finite second moments of the cycle lengths and centered rewards, with positive reward variance. Their bounds use truncated population moments and retain probabilities of long cycles and large rewards. Finally,~\cite{durmus2023rosenthal} derive Rosenthal-type moment bounds and Bernstein-type inequalities for additive functionals of geometrically ergodic Markov chains, with leading terms governed by the asymptotic variance and explicit dependence on the mixing time in the remainder terms.

Empirical-process results use regeneration to control a class of functions. For Markov chains with a stationary law, small-set minorization and exponential moment conditions on the initial and subsequent regeneration waiting times, \citet{Adamczak2008} proves concentration for bounded centered additive functionals. Under one-step minorization, they also control suprema over countable, uniformly bounded, centered classes, retaining their expected supremum and a population bound on the variances of cycle rewards. For positive Harris recurrent, aperiodic chains with one-step minorization, \citet{bertail_portier_2019} develop block Rademacher complexity bounds for uniformly bounded VC classes. Their bounds combine a uniform second-moment bound for block rewards with either a cycle moment of order greater than one or an exponential cycle moment. When combined with \cite{Adamczak2008}, this leads to novel Dvoretzky-Kiefer-Wolfowitz (DKW) type inequalities for exponentially regenerating Markov chains \cite{banerjee_chakrabarty_2025,banerjee_lei_mehrotra_2026}.

Concentration bounds can also retain random quantities from the trajectory. Under geometric drift, one-step minorization on Lyapunov sublevel sets, and well-defined conditional second moments, \citet{wintenberger_exponential_2017} proves an exponential inequality for path functionals satisfying a weighted bounded-difference condition. Its quadratic normalizer combines observed values of the Lyapunov function with transition-dependent conditional moments and coefficients. More directly, \citet{jerison_dkw_2026} gives a data-dependent Dvoretzky--Kiefer--Wolfowitz inequality, uniform over thresholds of one real-valued observable. They assume i.i.d. regeneration cycles, initialization at regeneration, identifiable boundaries, and a known exponential bound on the tail of the cycle length. The leading term uses observed cycle lengths, while a correction retains the tail constants. The sample contains a deterministic number of cycles, so its observation horizon is random.

The theory of generalization for machine learning algorithms was essentially developed for independent observations, see~\citet{mohri2012foundations,bach2024learning} for an introduction. However, there were notable attempts to extend some results to the setting of time series and stochastic processes, under various mixing and weak-dependence assumptions~\citep{yu1994rates,gamarnik1999extension,meir2000nonparametric,steinwart2009learning,mohri10a,modha2002memory,steinwart2009learning,shalizi2013predictive,kuznetsov2015learning,kuznetsov2017generalization,mcdonald2017nonparametric,kuznetsov2020discrepancy,abeles2024generalization}. Unfortunately, while the theory in the i.i.d. case allows to define to distribution free, agnostic criteria for training and model selection, all the results in the dependent case require some prior knowledge on the strength of the temporal dependence.
A class of generalization bounds known as PAC-Bayes bounds, introduced by~\cite{mcallester1998some} in the i.i.d. setting, which can be used both as criteria to optimize when training algorithms, and to derive powerful oracle inequalities~\cite{cat2007,alquier_user-friendly_2024,hellstrom2025generalization}. PAC-Bayes bounds were also extended to beyond independent observations:~\citet{ralaivola2010chromatic,fard2011pac,alq_wint_fast,e23030313,41f14ea790bf4ab0a2e153e453036d2b}. These results suffer the same limitation.

This led to a trend in the recent literature, which aims at estimating the unknown mixing coefficients involved in these bounds~\citep{khaleghi2023inferring,grunewalder2024estimating,wolfer2024optimistic,yao2026strongly}. For Markov chains,~\cite{wolfer2024,Karagulyan2} studied the estimation of the pseudo-spectral gap (in the case of finite, and compact continuous state-space, respectively). Together with Paulin's inequality~\cite{paulin2015}, this leads to agnostic deviation inequalities~\cite{karagulyan_empirical_2026,Karagulyan2}. Our results go beyond these works in many ways. First, out bound holds for general regenerative process, including non-Markovian ones. Then, even in the Markov case with a discrete state-space, our results are not limited to a finite state-space as in~\cite{wolfer2024,karagulyan_empirical_2026}: they can accommodate countable infinite state-spaces. Finally, our inequalities do not require the estimation of any mixing coefficient or spectral gap, which makes them easier to use in practice. It might seems that a drawback of our approach is that only the first order terms in our inequalities are empirical. In other words, we can derive asymptotic empirical confidence intervals from our bounds, but the non-asymptotic coverage is lost. We point out that this is also the case in~\cite{Karagulyan2}. In particular, this does not prevent from building efficient, agnostic learning rules.

\section{Statement of the problem and notations}
\label{sec:setup}

We summarize important notations in Table \ref{tab:notation}. Let \(\mathbb N:=\{1,2,\ldots\}\) and \(\mathbb N_0:=\{0,1,2,\ldots\}\). Let \(Z=(Z_t)_{t\in\mathbb N}\) take values in a measurable space \((\mathsf Z,\mathcal Z)\), and let \((\mathcal G_t)_{t\in\mathbb N_0}\) be a filtration to which \(Z\) is adapted. This filtration includes the variables used to identify regeneration times, such as the auxiliary marks in a split Markov chain.

\begin{definition}
\label{def:regenerative-process}
The process \(Z\) is regenerative if there are almost surely finite, integer-valued stopping times
\[
 0\leq\tau_1<\tau_2<\cdots
\]
such that, for every \(j\in\mathbb N\), the shifted process \((Z_{\tau_j+t})_{t\in\mathbb N}\) is independent of \(\mathcal G_{\tau_j}\) and has a law that does not depend on \(j\). We denote the common post-regeneration law by \(\mathbb P_A\) and its expectation by \(\mathbb E_A\).
\end{definition}

As previously mentioned, all Harris recurrent Markov chains can acquire a regenerative structure through Nummelin splitting \cite{bertail_clemencon_2006}. Furthermore, if \(X_t\) takes values in a measurable space \(\mathsf X\) and \(g:\mathsf X\to\mathsf Y\) is a fixed measurable map into a measurable space \(\mathsf Y\), then \((X_t,g(X_t))\) has the same regeneration times as \(X\). Non-Markovian regenerative models are popular in queueing and storage models \citep[Appendix~C]{meyn_markov_2012}.

For \(j\in\mathbb N\), define the cycle length and block by
\[
 \Delta_j:=\tau_{j+1}-\tau_j,
 \qquad
 B_j:=(Z_{\tau_j+1},\ldots,Z_{\tau_{j+1}}).
\]
Throughout the paper, we assume that the marked cycles \((\Delta_j,B_j)_{j\in\mathbb N}\) are i.i.d. with their common regenerative law, have positive integer lengths, and are independent of \(\tau_1\). We also assume
\[
 \mathbb E_A[\Delta_1]<\infty.
\]
We use \(\mathbb P_A\) and \(\mathbb E_A\) for the joint post-regeneration law and expectation, including the cycle marks. We observe the regeneration boundaries---i.e. we know the exact regeneration times---as well as the process. These marked-cycle assumptions specify the independence used in the results below; the definition of the post-regeneration observation process alone does not impose a common rule for marking the next cycle.

The process may start from an arbitrary initial distribution. Define its equilibrium occupation measure by, for any $C\in\mathcal Z$,
\[
 \bar\pi(C):=
 \frac{\mathbb E_A\left[
 \sum_{t=\tau_1+1}^{\tau_2}\mathbf 1_{\{Z_t\in C\}}\right]}
 {\mathbb E_A[\Delta_1]}.
\]
Given a (measurable) function $f:\mathsf Z\rightarrow [0,1]$, our objective is to provide empirical confidence bounds on
$$
\bar\pi(f) = \int f(z)\,\bar\pi(dz),
$$
from the observation of a single trajectory $(Z_1,\dots,Z_n)$ only. By this, we mean that the bounds should depend only on $(Z_1,\dots,Z_n)$, not on unknown parameters of the distribution of $Z$. As our applications will be mostly in machine learning, we will from work with loss function $f(z) = \ell_{\theta}(z) $, parametrized by $\theta\in\Theta$, with $0\leq f\theta(z) \leq 1 $. In this case, $\bar\pi(f)$ is usually written as a risk function
\[
 R(\theta):=\int\ell_\theta(z)\,\bar\pi(dz),
 \qquad\theta\in\Theta
\]
in the machine learning literature.
Until Section~\ref{sec:pac}, the parameter $\theta$ will be fixed and thus play little role. However, from Section~\ref{sec:pac}, we will discuss algorithms aiming at minimizing the risk $R(\theta)$.
While classical approaches rely on the empirical risk
$$ \hat{R}(\theta) := \frac{1}{n}\sum_{i=1}^n \ell_\theta(Z_i), $$
the regenerative structure suggests a different appraoch.
For \(j\in\mathbb N\) and \(\theta\in\Theta\), define the block loss and centered discrepancy by
\[
 L_{j,\theta}:=\sum_{t=\tau_j+1}^{\tau_{j+1}}\ell_\theta(Z_t),
 \qquad \xi_{j,\theta}:=\Delta_jR(\theta)-L_{j,\theta}.
\]
The regenerative form of Kac's lemma \cite[Theorem~10.4.9]{meyn_markov_2012} gives the risk identity below. This equality centers the block discrepancies and underlies our bounds (for example, it enables the exponential supermartingale construction in Lemma~\ref{lem:stopped-exponential}):
\[
 R(\theta)=\frac{\mathbb E_A[L_{1,\theta}]}
 {\mathbb E_A[\Delta_1]},
 \qquad \mathbb E_A[\xi_{j,\theta}]=0,
 \qquad |\xi_{j,\theta}|\leq\Delta_j.
\]
Given the sample size \(n\in\mathbb N\), let
\[
 M_n:=\min\{m\in\mathbb N:\tau_{m+1}>n\}.
\]
This is the index of the first cycle whose end is after time \(n\). The blocks \(B_1,\ldots,B_{M_n-1}\) are complete. If \(\tau_1\leq n\), then we also have
\[
 M_n=\max\{m\in\mathbb N:\tau_m\leq n\}.
\]
Let \(\mathcal F_0:=\sigma(\tau_1)\) and \(\mathcal F_m:=\sigma(\tau_1,B_1,\ldots,B_m)\) for \(m\in\mathbb N\). Define
\[
 D_n:=\tau_{M_n}-\tau_1
 =\sum_{j=1}^{M_n-1}\Delta_j.
\]
On \(\{D_n>0\}\), the completed-block empirical risk is
\[
 \widehat R_n^{\mathrm{reg}}(\theta):=
 \frac1{D_n}\sum_{j=1}^{M_n-1}L_{j,\theta}.
\]
We set it to zero when \(D_n=0\), except where we explicitly specify a different convention. The length \(n-D_n\) is the part of the sample outside the completed cycles.

As mentioned above, a difficulty here is that a bounded loss can have an unbounded block sum because cycle lengths can be unbounded. Various techniques were introduce to prove PAC and PAC-Bayes bounds for unbounded losses. One of the simplest idea is truncation, as used by~\citet{alquier_pac-bayesian_2008}, and also sufficient in our context. We compare the average of the $\xi_{j,\theta}$ to an average of truncated variables $Y_{j,\theta}:=\xi_{j,\theta}\mathbf 1_{\{\Delta_j\leq u_n\}}$ for some finite deterministic truncation level \(u_n\in(0,\infty)\).  Put, for any $u>0$,
\[
 \beta(u):=\mathbb E_A[\Delta_1\mathbf 1_{\{\Delta_1>u\}}].
\]
Finally, put
\[
 q(x):=-\log(1-x)-x,\qquad x\in[0,1).
\]
It will be convenient to extend $q(x)=0$ for $x>1$.

\begin{table}[htbp]
\centering
\caption{Notation used throughout the paper.}
\label{tab:notation}
\small
\renewcommand{\arraystretch}{1.15}
\begin{tabularx}{\textwidth}{@{}
>{\raggedright\arraybackslash}p{0.14\textwidth}
>{\raggedright\arraybackslash}X
>{\raggedright\arraybackslash}p{0.14\textwidth}
>{\raggedright\arraybackslash}X@{}}
\toprule
Notation & Definition & Notation & Definition \\
\midrule
\(Z_t\)
& Observation at time \(t\).
& \(\mathcal G_t\)
& Filtration including observations and regeneration marks. \\

\(\tau_j\)
& The \(j\)-th regeneration time.
& \(\mathbb P_A,\mathbb E_A\)
& Common post-regeneration law and expectation. \\

\(\Delta_j\)
& Block length \(\tau_{j+1}-\tau_j\).
& \(B_j\)
& Block \((Z_{\tau_j+1},\ldots,Z_{\tau_{j+1}})\). \\

\(\Theta,\Pcal(\Theta)\)
& Parameter space and its probability measures.
& \(\mu,\rho\)
& Prior and generic posterior on \(\Theta\). \\

\(\ell_\theta\)
& Loss associated with \(\theta\), taking values in \([0,1]\).
& \(\bar\pi\)
& Equilibrium occupation measure. \\

\(R(\theta)\)
& Long-run risk \(\int\ell_\theta(z)\,\bar\pi(dz)\).
& \(L_{j,\theta}\)
& Cumulative loss in block \(j\). \\

\(\xi_{j,\theta}\)
& Centered discrepancy \(\Delta_jR(\theta)-L_{j,\theta}\).
& \(n\)
& Length of the observed sample path. \\

\(M_n\)
& Index of the first cycle ending after \(n\).
& \(\mathcal F_m\)
& Block filtration \(\sigma(\tau_1,B_1,\ldots,B_m)\). \\

\(D_n\)
& Total length of the completed blocks.
& \(\widehat R_n^{\mathrm{reg}}(\theta)\)
& Empirical risk over the completed blocks. \\

\(u_n\)
& Deterministic truncation level.
& \(\beta(u)\)
& Tail contribution \(\mathbb E_A[\Delta_1\mathbf 1_{\{\Delta_1>u\}}]\). \\

\(Y_{j,\theta}\)
& Truncated discrepancy \(\xi_{j,\theta}\mathbf 1_{\{\Delta_j\leq u_n\}}\).
& \(q(x)\)
& \(-\log(1-x)-x\), for \(x\in[0,1)\). \\
\bottomrule
\end{tabularx}
\end{table}

\section{Self-normalized Bernstein inequalities for regenerative processes}
\label{sec:bernstein}

We begin with the statement of our self-normalized Bernstein inequality.
\begin{theorem}
\label{thm:regenerative-bernstein}
Fix \(\theta\in\Theta\), deterministic values \(u_n>0\) and \(0<\lambda<n/u_n\), and \(\delta\in(0,1)\). With probability at least \(1-\delta\), on \(\{D_n>0\}\),
\begin{align}
\left|R(\theta)-\widehat R_n^{\mathrm{reg}}(\theta)\right|
\leq{}&\frac{n}{\lambda D_n}\Bigg[
\log\frac2\delta+M_n\log\left(1+\frac{\lambda\beta(u_n)}n\right)
\notag\\
&\qquad+\sum_{j=1}^{M_n-1}\mathbf 1_{\{\Delta_j\leq u_n\}}
q\left(\frac{\lambda\Delta_j}n\right)
-\log\left(1-\frac{\lambda u_n}n\right)\Bigg]
\notag\\
&+\frac1{D_n}\sum_{j=1}^{M_n-1}
\Delta_j\mathbf 1_{\{\Delta_j>u_n\}}.
\label{eq:regenerative-bernstein}
\end{align}
\end{theorem}
By convention, we will consider that the bound (and all its corollaries below) are defined as \(+\infty\) when \(D_n=0\).
\begin{remark}
The only nonempirical term in \eqref{eq:regenerative-bernstein} is
\[
\beta(u_n)
=
\mathbb E_A\left[
\Delta_1\mathbf 1_{\{\Delta_1>u_n\}}
\right].
\]
Under exponential tails, the choices of \(u_n\) and \(\lambda\) in Section~\ref{sec:regimes} make its contribution to the expected deviation bound \(O((\log n)/n)=o(n^{-1/2})\), smaller than the leading \(n^{-1/2}\) term. This need not hold under polynomial tails with exponent \(1<\alpha<2\). Moreover, it doesn't depend on $\theta$. Thus, in practical machine learning applications where we will use this inequality to define learning objectives, it is actually safe to ignore this term.
\end{remark}

Theorem~\ref{thm:regenerative-bernstein} uses a deterministic tuning parameter $\lambda$. We note that such a parameter is common in PAC-Bayes bounds, where it is actually useful to keep this parameter free~\cite{alquier_user-friendly_2024}. However, we show here that we can optimize this parameter to get our first-order empirical version of Bernstein inequality. Define
\[
J_n:=2\left(1+\left\lceil\log_2\max\{1,n/u_n\}\right\rceil\right).
\]
\begin{corollary}
\label{cor:empirical-bernstein}
Fix \(\theta\in\Theta\), a deterministic \(u_n>0\), and \(\delta\in(0,1)\). Put
\[
H_{n,u}:=\sum_{j=1}^{M_n-1}\Delta_j^2\mathbf 1_{\{\Delta_j\leq u_n\}},
\qquad
T_{n,u}:=\sum_{j=1}^{M_n-1}\Delta_j\mathbf 1_{\{\Delta_j>u_n\}},
\]
Then, with probability at least \(1-\delta\), on \(\{D_n>0\}\),
\begin{equation}
\left|R(\theta)-\widehat R_n^{\mathrm{reg}}(\theta)\right|
\leq\frac{3\sqrt{H_{n,u}\log(J_n/\delta)}
+2u_n[\log(J_n/\delta)+1]}{D_n}+\frac{M_n\beta(u_n)+T_{n,u}}{D_n}.
\label{eq:empirical-bernstein}
\end{equation}
\end{corollary}
% This is essentially obtained by a union bound on $\lambda$ in a well-chosen grid of cardinality $J_n$. \imon{mark the proof subsection}

We next integrate the deviation inequality. The resulting bound separates the truncated second moment, the tail contribution, and the expected length outside the completed cycles. This allows to understand the order of magnitude of the empirical deviations above.
\begin{theorem}
\label{thm:expectation-deviation}
Fix \(n\geq1\), \(\theta\in\Theta\), and deterministic values \(u_n>0\) and \(0<\lambda\leq n/(2u_n)\). Set \(\widehat R_n^{\mathrm{reg}}(\theta)=0\) when \(D_n=0\), and write, for any $u>0$,
\[
v(u):=\mathbb E_A\left[\Delta_1^2\mathbf 1_{\{\Delta_1\leq u\}}\right].
\]
Then
\begin{align}
\mathbb E\left[
\left|R(\theta)-\widehat R_n^{\mathrm{reg}}(\theta)\right|
\right]
\leq&\frac{1+\log2}{\lambda}
+\frac{\lambda\mathbb E[M_n]}{n^2}v(u_n)
\notag\\
&+\frac{2\mathbb E[M_n]}n\beta(u_n)
+\frac{2u_n}n+\frac{n-\mathbb E[D_n]}n.
\label{eq:expectation-deviation}
\end{align}
\end{theorem}
In the next section, we will make all the terms in~\eqref{eq:expectation-deviation} ---$ \mathbb E[D_n]$, $\mathbb E[M_n]$, $\beta(u_n)$ and $v(u_n)$---explicit under various assumptions on the regeneration. In particular, the ratio $r_n=\tfrac{\mathbb E[M_n]}n $ is a key quantity. As a proof of concept, let us just explicit these when the data process is made of blocks of length at most $k$, independent of each other. That is, $1\leq \Delta_j = \tau_{j+1} - \tau_j \leq k $. Also suppose $\tau_1=1$. This includes i.i.d. data for $k=1$. We always have $M_n \leq n+1 $, so $r_n\leq (n+1)/n \leq 2$, moreover, in this case, $D_n \leq 2k$, $v(u_n) \leq k^2 $ and thus~\eqref{eq:expectation-deviation} gives
\begin{equation*}
\mathbb E\left[
\left|R(\theta)-\widehat R_n^{\mathrm{reg}}(\theta)\right|
\right]
\leq \frac{1+\log2}{\lambda}
+\frac{2 \lambda k^2 }{n}
+4  \beta(u_n) 
+\frac{2u_n}n+\frac{2k}n.
\end{equation*}
We can simply put $u_n = k$ to get $\beta(u_n)=0$, and an exact minimization in $\lambda$ gives $\lambda = (1/k)\sqrt{n/[2(1+\log 2)]} $, which satisfies $0\leq \lambda \leq n/(2u_n)=n/(2k)$ for $n$ large enough, in this case:
\begin{equation*}
\mathbb E\left[
\left|R(\theta)-\widehat R_n^{\mathrm{reg}}(\theta)\right|
\right]
\leq 2\sqrt{ \frac{2k(1+\log 2)}{n} }
+\frac{4k}n.
\end{equation*}
More generally, in situations where $ \Delta_j$ is unbounded but has light enough tails, we expect to recover $\mathbb E(M_n) = \mathcal{O}(n)$, $\mathbb E(D_n) = \mathcal{O}(1)$ and, for $u_n\rightarrow \infty$ well chosen, $ \beta(u_n) = o(1/\sqrt{n}) $ and $u_n/n =  o(1/\sqrt{n}) $.

\section{Control of the term $\beta(u_n)$ in various regimes}
\label{sec:regimes}

This section shows how assumptions on regeneration lengths yield explicit bounds and convergence rates for the expected estimation error. We first describe how Markov chains give rise to regenerative blocks, then specialize Theorem~\ref{thm:expectation-deviation} to independent observations and to regeneration lengths with exponential or polynomial tails. From now, we keep the notation
$$ r_n:=\mathbb E[M_n]/n\leq(n+1)/n.$$

\subsection{Regeneration for Markov chains}

A Markov chain regenerates at successive returns to an accessible atom \cite[Chapter~5]{meyn_markov_2012}.

\begin{definition}
\label{def:atomic}
A \(\psi\)-irreducible, aperiodic Markov chain \(X=(X_m)_{m\in\mathbb N_0}\) with transition kernel \(P\) is \emph{atomic} if there is a measurable set \(A\), called an atom, with \(\psi(A)>0\), such that
\[
P(x,\cdot)=P(y,\cdot),\qquad x,y\in A.
\]
\end{definition}
\begin{remark}
In a countable irreducible state space, any singleton may serve as an atom.
\end{remark}
\begin{remark}
For a chain without an accessible atom, Nummelin splitting constructs one from a small-set minorization. Suppose that an accessible set \(S\) satisfies, for some \(r\in\mathbb N\), \(\epsilon_S\in(0,1)\), and probability measure \(\Gamma_S\),
\begin{equation*}
P^r(x,C)\geq\epsilon_S\Gamma_S(C),\qquad x\in S,
\end{equation*}
for every measurable \(C\). For \(r=1\), introduce independent Bernoulli variables \((Y_m)_{m\geq1}\) with success probability \(\epsilon_S\). At a visit \(X_m\in S\), draw \(X_{m+1}\) from \(\Gamma_S\) if \(Y_m=1\), and otherwise from the residual kernel
\[
\frac{P(X_m,\cdot)-\epsilon_S\Gamma_S(\cdot)}{1-\epsilon_S}.
\]
This preserves the transition kernel of \(X\), and the split chain \(Z_m=(X_m,Y_m)\) has atom \(A=S\times\{1\}\). For \(r>1\), the construction applies to the \(r\)-step skeleton; applying the bounds to all intermediate observations requires a separately verified i.i.d. cycle construction \cite[Section~1.2.3]{bertail_clemencon_2006}. The problem is that, we need to know $S$ and $\varepsilon_S$ to use this in practice. We refer the reader to~\citet[p.~13]{bertail_clemencon_2006} who propose a procedure to estimate such a set $S$.
\end{remark}
Define the successive visits to the atom by
\[
\rho_A(1):=\inf\{m\geq1:Z_m\in A\} \text{, and, for }j\geq2\text{, }
\rho_A(j):=\inf\{m>\rho_A(j-1):Z_m\in A\}.
\]
Assuming these times are almost surely finite, the strong Markov property gives i.i.d. blocks
\[
B_j=(X_{\rho_A(j)+1},\ldots,X_{\rho_A(j+1)}),\qquad j\geq1,
\]
with i.i.d. lengths \(\rho_A(j+1)-\rho_A(j)\). We require their common mean to be finite, as in Section~2.

\subsection{Expected deviations under regeneration-length assumptions}

We now make the tail term \(\beta(u_n)\) explicit under exponential and polynomial assumptions on the regeneration lengths, and we specify a truncation level. While this is essentially all we need for the applications in machine learning in Section~\ref{sec:pac}, we also apply Theorem~\ref{thm:expectation-deviation}, and minimize the resulting bound over the admissible values of \(\lambda\) to make the order of magnitude of the bounds explicit.
The proofs use only the marked-cycle assumptions, so the same conclusions apply to the non-Markovian example in the next subsection.

Throughout this subsection, fix \(\theta\in\Theta\), assume \(\tau_1=0\) (which always holds if we are ready to throw away an initial segment of the observed trajectory).
We study exponential regeneration tails in a first time, and polynomial ones in a second time.

\begin{corollary}
\label{cor:deviation-exponential}
\label{prop:oracle-exponential}
\label{cor:deviation-exponential-rate}
\label{cor:oracle-exponential-asymptotic}
Suppose that, for some \(C_{\mathrm e},\kappa>0\),
\[
\mathbb P_A(\Delta_1>t)\leq C_{\mathrm e}e^{-\kappa t},\qquad t\geq1.
\]
For every integer \(u\geq1\),
\begin{equation}
\beta(u)\leq C_{\mathrm e}e^{-\kappa u}
\left(u+\frac1{1-e^{-\kappa}}\right).
\label{eq:exponential-beta}
\end{equation}
Moreover, the second moment and the expected renewal age satisfy
\[
v_2:=\mathbb E_A[\Delta_1^2]
\leq1+\frac{C_{\mathrm e}e^{-\kappa}(3-e^{-\kappa})}
{(1-e^{-\kappa})^2},
\qquad
\mathbb E[n-D_n]\leq
\frac{C_{\mathrm e}e^{-\kappa}}{(1-e^{-\kappa})^2}.
\]
For \(n\geq1\), choose
\[
u_n:=\max\left\{1,\left\lceil\frac{\log n}{\kappa}\right\rceil\right\}.
\]
Then
\begin{equation}
\beta(u_n)\leq\frac{C_{\mathrm e}}n
\left(u_n+\frac1{1-e^{-\kappa}}\right)
\leq\frac{C_{\mathrm e}}n
\left(\frac{\log n}{\kappa}+1+\frac1{1-e^{-\kappa}}\right).
\label{eq:exponential-beta-un}
\end{equation}
For every deterministic \(0<\lambda\leq n/(2u_n)\),
\begin{align}
\mathbb E\left[\left|R(\theta)-\widehat R_n^{\mathrm{reg}}(\theta)\right|\right]
\leq{}&\frac{1+\log2}{\lambda}+\frac{r_n\lambda v_2}n
\notag\\
&+2r_n C_{\mathrm e}e^{-\kappa u_n}
\left(u_n+\frac1{1-e^{-\kappa}}\right)
\notag\\
&+\frac{2u_n}n+\frac{C_{\mathrm e}e^{-\kappa}}{n(1-e^{-\kappa})^2}.
\label{eq:exponential-unoptimized}
\end{align}
The right-hand side is minimized over this interval by
\begin{equation}
\lambda_n:=\min\left\{
\sqrt{\frac{(1+\log2)n}{r_nv_2}},\frac n{2u_n}\right\}.
\label{eq:exponential-lambda}
\end{equation}
This choice yields the explicit bound
\begin{align}
\mathbb E\left[\left|R(\theta)-\widehat R_n^{\mathrm{reg}}(\theta)\right|\right]
\leq{}&2\sqrt{\frac{(1+\log2)r_nv_2}n}
+\frac{2(2+\log2)u_n}n
\notag\\
&+\frac{2r_n C_{\mathrm e}}n
\left(u_n+\frac1{1-e^{-\kappa}}\right)
+\frac{C_{\mathrm e}e^{-\kappa}}{n(1-e^{-\kappa})^2}.
\label{eq:exponential-order-bound}
\end{align}
For a fixed regenerative law, \(\lambda_n\sim\sqrt{(1+\log2)mn/v_2}\), and
\begin{equation}
\mathbb E\left[\left|R(\theta)-\widehat R_n^{\mathrm{reg}}(\theta)\right|\right]
\leq\left(2\sqrt{\frac{(1+\log2)v_2}m}+o(1)\right)n^{-1/2}
+O\left(\frac{\log n}n\right).
\label{eq:exponential-unoptimized-asymptotic}
\end{equation}
\end{corollary}

Here \(\beta(u_n)=O((\log n)/n)=o(n^{-1/2})\). Thus the contribution \(2r_n\beta(u_n)\), the truncation correction \(2u_n/n\), and the expected renewal age divided by \(n\) are all second order relative to the \(n^{-1/2}\) deviation rate.

Under polynomial tails, the truncated second moment and the role of the tail contribution depend on whether the tail exponent is above, equal to, or below two.

\begin{corollary}
\label{cor:deviation-polynomial}
\label{prop:oracle-polynomial}
\label{cor:deviation-polynomial-rates}
\label{cor:oracle-polynomial-asymptotic}
Suppose that, for some \(C_{\mathrm p}>0\) and \(\alpha>1\),
\[
\mathbb P_A(\Delta_1>t)\leq C_{\mathrm p}t^{-\alpha},\qquad t\geq1.
\]
Put \(B_\alpha:=\alpha C_{\mathrm p}/(\alpha-1)\). For every \(u\geq1\),
\begin{equation}
\beta(u)\leq B_\alpha u^{1-\alpha}.
\label{eq:polynomial-beta}
\end{equation}
For \(n\geq1\), choose \(u_n:=\lceil n^{1/\alpha}\rceil\). Then
\begin{equation}
\beta(u_n)\leq B_\alpha u_n^{1-\alpha}
\leq B_\alpha n^{-(\alpha-1)/\alpha}.
\label{eq:polynomial-beta-un}
\end{equation}
For \(u\geq1\), define
\[
V_\alpha(u):=
\begin{cases}
\displaystyle 1+\frac{2C_{\mathrm p}}{\alpha-2},&\alpha>2,\\[2mm]
\displaystyle 1+2C_{\mathrm p}\log u,&\alpha=2,\\[2mm]
\displaystyle \left(1+\frac{2C_{\mathrm p}}{2-\alpha}\right)u^{2-\alpha},
&1<\alpha<2,
\end{cases}
\]
and put
\[
G_\alpha(n):=
\begin{cases}
\displaystyle C_{\mathrm p}\left(1+\frac1{\alpha-2}\right),&\alpha>2,\\[2mm]
\displaystyle C_{\mathrm p}(1+\log n),&\alpha=2,\\[2mm]
\displaystyle C_{\mathrm p}\left(1+\frac{n^{2-\alpha}}{2-\alpha}\right),
&1<\alpha<2.
\end{cases}
\]
These quantities satisfy \(v(u)\leq V_\alpha(u)\) and \(\mathbb E[n-D_n]\leq G_\alpha(n)\). For every deterministic \(0<\lambda\leq n/(2u_n)\),
\begin{align}
\mathbb E\left[\left|R(\theta)-\widehat R_n^{\mathrm{reg}}(\theta)\right|\right]
\leq{}&\frac{1+\log2}{\lambda}
+\frac{r_n\lambda}nV_\alpha(u_n)
\notag\\
&+2r_nB_\alpha u_n^{1-\alpha}
+\frac{2u_n+G_\alpha(n)}n.
\label{eq:polynomial-explicit-bound}
\end{align}
The right-hand side is minimized over this interval by
\begin{equation}
\lambda_n:=\min\left\{
\sqrt{\frac{(1+\log2)n}{r_nV_\alpha(u_n)}},\frac n{2u_n}\right\}.
\label{eq:polynomial-lambda}
\end{equation}
This choice yields
\begin{align}
\mathbb E\left[\left|R(\theta)-\widehat R_n^{\mathrm{reg}}(\theta)\right|\right]
\leq{}&2\sqrt{\frac{(1+\log2)r_nV_\alpha(u_n)}n}
+\frac{2(2+\log2)u_n}n
\notag\\
&+2r_nB_\alpha u_n^{1-\alpha}
+\frac{G_\alpha(n)}n.
\label{eq:polynomial-deviation-optimized}
\end{align}
Keeping the regenerative law fixed as \(n\to\infty\), we obtain
\begin{align}
\mathbb E\left[\left|R(\theta)-\widehat R_n^{\mathrm{reg}}(\theta)\right|\right]
&=O(n^{-1/2}),&&\alpha>2,
\label{eq:polynomial-alpha-greater-two}\\
\mathbb E\left[\left|R(\theta)-\widehat R_n^{\mathrm{reg}}(\theta)\right|\right]
&=O\left(\sqrt{\frac{\log n}n}\right),&&\alpha=2,
\label{eq:polynomial-alpha-equal-two}\\
\mathbb E\left[\left|R(\theta)-\widehat R_n^{\mathrm{reg}}(\theta)\right|\right]
&=O\left(n^{-(\alpha-1)/\alpha}\right),&&1<\alpha<2.
\label{eq:polynomial-alpha-less-two}
\end{align}
The constants in these order terms may depend on the fixed regenerative law, \(C_{\mathrm p}\), and \(\alpha\).
\end{corollary}

\begin{remark}
For \(\alpha>2\), \eqref{eq:polynomial-beta-un} gives \(\beta(u_n)=o(n^{-1/2})\), so the tail contribution is second order. At \(\alpha=2\), it gives \(\beta(u_n)=O(n^{-1/2})\), which is still smaller than the general rate \(\sqrt{(\log n)/n}\), but only by a logarithmic factor.

For \(1<\alpha<2\), the available bound on \(\beta(u_n)\) has the same order \(n^{-(\alpha-1)/\alpha}\) as the optimized quadratic and truncation terms. The tail contribution is therefore no longer necessarily second order. In particular, if \(\mathbb P_A(\Delta_1>t)\sim c\,t^{-\alpha}\) for some \(c>0\), then
\[
\beta(u_n)\sim\frac{\alpha c}{\alpha-1}
n^{-(\alpha-1)/\alpha}.
\]
Since \(r_n\to1/m\), its contribution \(2r_n\beta(u_n)\) is then of leading order in \eqref{eq:polynomial-deviation-optimized}. Consequently, it cannot in general be omitted when assessing the leading magnitude of the bound. The polynomial upper-tail assumption alone does not force this behaviour, since it also allows more rapidly decaying tails.
\end{remark}

\begin{remark}
The tail assumption with \(\alpha=2\) does not imply \(v_2:=\mathbb E_A[\Delta_1^2]<\infty\). If this second moment is finite, \(V_2(u_n)\) may be replaced by \(v_2\) in \eqref{eq:polynomial-explicit-bound}, \eqref{eq:polynomial-lambda}, and \eqref{eq:polynomial-deviation-optimized}. With the same choice \(u_n=\lceil\sqrt n\rceil\), this improves the rate to \(O(n^{-1/2})\). In this case the actual tail term satisfies \(\beta(u_n)=o(n^{-1/2})\); the coarser upper bound in \eqref{eq:polynomial-beta-un} need not reflect this improvement. The correction \(u_n/n\), however, remains of order \(n^{-1/2}\).
\end{remark}

\subsection{A non-Markovian example}

We next apply the same calculation to a queue whose length process is not Markovian. Returns to the empty state give independent cycles. After verifying exponential tails for their lengths, we use Corollary~\ref{cor:deviation-exponential} to obtain explicit tail and expected deviation bounds.

\paragraph{Example: a Geo/D/1 queue.}
Let \((I_t)_{t\in\mathbb N_0}\) be independent Bernoulli\((p)\) arrival indicators, where \(p\in(0,1/2)\). Start the queue empty, with \(Q_0=0\), and let \(Q_t\in\mathbb N_0\) be the number of customers immediately before the arrival in slot \(t\). An arriving customer joins the queue at the beginning of the slot. The server then provides one slot of service, in order of arrival and without preemption. Each customer requires two service slots, and departures occur at the end of a slot. A customer arriving to an empty system starts service in that slot; after a departure, the next waiting customer starts service in the following slot.

Define the successive empty observations by
\[
 \tau_1:=0,\qquad
 \tau_{j+1}:=\inf\{t>\tau_j:Q_t=0\},\qquad j\in\mathbb N.
\]
An empty slot with no arrival gives a cycle of length one; an arrival to the empty system starts a busy-period cycle. We use the same block convention as in Section~\ref{sec:setup}:
\[
 \Delta_j:=\tau_{j+1}-\tau_j,
 \qquad B_j:=(Q_{\tau_j+1},\ldots,Q_{\tau_{j+1}}).
\]

\begin{corollary}
\label{cor:deviation-queue}
\label{lem:nonmarkov-queue}
The queue-length process \((Q_t)_{t\in\mathbb N_0}\) is not Markov with respect to its natural filtration, but the marked cycles \((\Delta_j,B_j)_{j\in\mathbb N}\) are i.i.d. We define $\kappa_0:=-\frac12\log(4p(1-p)), C_0:=e^{2\kappa_0}.$
Then \(\kappa_0>0\) and
\[
 \mathbb P(\Delta_1>t)\leq C_0e^{-\kappa_0t},\qquad t\geq0.
\]
In particular, all regeneration times are almost surely finite, the cycle lengths have finite moments of every order, and
\begin{equation}
m:=\mathbb E[\Delta_1]=\frac{1-p}{1-2p},
\qquad
v_2:=\mathbb E[\Delta_1^2]
=\frac{(1-p)(1-2p+4p^2)}{(1-2p)^3}.
\label{eq:queue-cycle-moments}
\end{equation}
For every integer \(u\geq1\),
\begin{equation}
\beta(u)\leq C_0e^{-\kappa_0u}
\left(u+\frac1{1-e^{-\kappa_0}}\right).
\label{eq:queue-beta}
\end{equation}
For \(n\geq1\), choose
\[
u_n:=\max\left\{1,\left\lceil\frac{\log n}{\kappa_0}\right\rceil\right\}.
\]
Then
\begin{equation}
\beta(u_n)\leq\frac{C_0}n
\left(u_n+\frac1{1-e^{-\kappa_0}}\right)
\leq\frac{C_0}n
\left(\frac{\log n}{\kappa_0}+1+\frac1{1-e^{-\kappa_0}}\right).
\label{eq:queue-beta-un}
\end{equation}

Fix \(\theta\in\Theta\), take \(Z_t=Q_t\)). For every deterministic \(0<\lambda\leq n/(2u_n)\),
\begin{align}
\mathbb E\left[\left|R(\theta)-\widehat R_n^{\mathrm{reg}}(\theta)\right|\right]
\leq{}&\frac{1+\log2}{\lambda}+\frac{r_n\lambda v_2}n
\notag\\
&+2r_n C_0e^{-\kappa_0u_n}
\left(u_n+\frac1{1-e^{-\kappa_0}}\right)
\notag\\
&+\frac{2u_n}n+\frac{C_0e^{-\kappa_0}}{n(1-e^{-\kappa_0})^2}.
\label{eq:queue-deviation-unoptimized}
\end{align}
The right-hand side is minimized over this interval by
\begin{equation}
\lambda_n:=\min\left\{
\sqrt{\frac{(1+\log2)n}{r_nv_2}},\frac n{2u_n}\right\}.
\label{eq:queue-deviation-lambda}
\end{equation}
This choice yields
\begin{align}
\mathbb E\left[\left|R(\theta)-\widehat R_n^{\mathrm{reg}}(\theta)\right|\right]
\leq{}&2\sqrt{\frac{(1+\log2)r_nv_2}n}
+\frac{2(2+\log2)u_n}n
\notag\\
&+\frac{2r_n C_0}n
\left(u_n+\frac1{1-e^{-\kappa_0}}\right)
+\frac{C_0e^{-\kappa_0}}{n(1-e^{-\kappa_0})^2}.
\label{eq:queue-deviation-optimized}
\end{align}
For fixed \(p\), \(r_n\to(1-2p)/(1-p)\), and hence
\begin{align}
\mathbb E\left[\left|R(\theta)-\widehat R_n^{\mathrm{reg}}(\theta)\right|\right]
\leq{}&
\left(\frac{2\sqrt{(1+\log2)(1-2p+4p^2)}}{1-2p}
+o(1)\right)n^{-1/2}
\notag\\
&+O\left(\frac{\log n}n\right).
\label{eq:queue-deviation-rate}
\end{align}
The constants in the order terms depend only on \(p\).
\end{corollary}

Thus this non-Markovian process has an expected deviation bound of order \(n^{-1/2}\), with the explicit tail contribution \(\beta(u_n)=O((\log n)/n)\) again of second order. Since \(v_2\) is explicit in \(p\), replacing \(r_n\) by \((n+1)/n\) in the finite-sample bound and its minimizing choice of \(\lambda\) also gives a version involving only \(p\) and \(n\).

\section{Applications to classification with PAC-Bayes bounds}
\label{sec:pac}

The preceding bounds concern a fixed parameter. We now consider the situation where $\ell_\theta(z)$ measures the classification error of a parameter $\theta$ on a data $z$. In this case $R(\theta)$ is the generalization risk. The objective is to build a classifier whose classification error is as close as possible to $\inf_{\theta\in\Theta}R(\theta) $. PAC-Bayes bounds consider randomized classification procedure, that is, sample-dependent probability distributions over $\Theta$. This requires a measurable space structure. So, from now, \((\Theta,\mathcal T)\) is a measurable parameter space and let \(\Pcal(\Theta)\) will denote the set of its probability measures. Assume that
\[
 (\theta,z)\longmapsto\ell_\theta(z),
 \qquad \Theta\times\mathsf Z\longrightarrow[0,1],
\]
is jointly measurable. For the learning results, fix a prior \(\mu\in\Pcal(\Theta)\) independently of the observations. We use \(\rho\in\Pcal(\Theta)\) for a posterior distribution; a posterior chosen from the data is a measurable probability kernel. Write \(\KL(\rho\Vert\mu)\) for the  relative entropy or Kullback-Leibler divergence:
\[
 \KL(\rho\Vert\mu):=
 \int_\Theta\log\!\left(\frac{d\rho}{d\mu}\right)\,d\rho
 \quad\text{if }\rho\ll\mu,
\]
and set it to \(+\infty\) otherwise.

\subsection{PAC-Bayes bounds}

\begin{theorem}
\label{prop:initial-PACBayes}
For every deterministic
\[
0<\lambda<\frac{n}{u_n}
\]
and every \(\delta\in(0,1)\), with probbility at least $1-\delta$ over the observed sample $(Z_1,\dots,Z_n)$, simultaneously for every $ \rho\in\Pcal(\Theta) $,
\begin{align}
\mathbb E_{\theta\sim\rho}R(\theta)
\leq &
\mathbb E_{\theta\sim\rho}
\widehat R_n^{\mathrm{reg}}(\theta)
\notag\\
&+
\frac{n}{\lambda D_n}
\Bigg[
\KL(\rho\Vert\mu)
+\log\left(\frac{1}{\delta}\right)
+M_n\log\left(
1+\frac{\lambda\beta(u_n)}{n}
\right)
\notag\\
&\qquad\qquad
+\sum_{j=1}^{M_n-1}
\mathbf 1_{\{\Delta_j\leq u_n\}}
q\left(\frac{\lambda\Delta_j}{n}\right)
-\log\left(1-\frac{\lambda u_n}{n}\right)
\Bigg]+
\frac{1}{D_n}
\sum_{j=1}^{M_n-1}
\Delta_j\mathbf 1_{\{\Delta_j>u_n\}}.
\label{eq:main-bound}
\end{align}
\end{theorem}
The right-hand side can be used as a learning objective. The minimization of PAC-Bayes bounds with respect to $\rho$ leads to randomized classifiers $\hat{\rho}_\lambda$ that satisfy oracle inequalities~\citep{cat2007,alquier_user-friendly_2024}. In practice, these classifiers also have very good performance, we refer the reader to~\cite{dzi2017} for application to image classification with deep neural networks.

The important point is that the minimization in $\rho$ does not involve any term depending on distribution of the tails of the $\Delta_j$'s. Actually, it involves only observable (empirical) quantities:
$$
\hat{\rho}_\lambda := \arg\min_{\rho\in \Pcal(\Theta)} 
\left\{
\mathbb E_{\theta\sim\rho}
\widehat R_n^{\mathrm{reg}}(\theta) + \frac{n \KL(\rho\Vert\mu) }{\lambda D_n}
\right\}.
$$
In the previous attempts on empirical PAC-Bayes bounds for non i.i.d. data~\cite{karagulyan_empirical_2026}, the minimization in $\rho$ actually involved non-empirical terms, that were estimated.

From Section 2.1.2 page 20 of~\cite{alquier_user-friendly_2024}, we have an explicit form for $\hat{\rho}_\lambda$, given by its density in \(\theta\in\Theta\) with respect to $\mu$:
\begin{equation}
 \frac{d\widehat\rho_\lambda}{d\mu}(\theta)
 =\frac{\exp\left\{-\frac\lambda n\sum_{j=1}^{M_n-1}L_{j,\theta}\right\}}
 {\displaystyle\int_\Theta
 \exp\left\{-\frac\lambda n\sum_{j=1}^{M_n-1}L_{j,\vartheta}\right\}
 \,\mu(d\vartheta)}.
 \label{eq:gibbs-posterior}
\end{equation}
Note, however, that the randomized estimator depends on $\lambda$. Thus, we still have to tune $\lambda$, but this time without using any information on the data-generating process.

\subsection{Oracle risk bound in expectation}

In order to allow an adaptive choice for $\lambda$, we will propose an optimization over a finite grid $\Lambda_n$ (the same strategy was used to derive Corollary~\ref{cor:empirical-bernstein} from Theorem~\ref{thm:regenerative-bernstein}). Thus, let $\Lambda_n\subset(0,n/u_n) $ be non-empty, finite, with $K_n =|\Lambda_n| $ and $\lambda_+:=\max\Lambda_n$. 

For each \(\lambda\in\Lambda_n\), keep the Gibbs posterior \(\widehat\rho_\lambda\) from \eqref{eq:gibbs-posterior}, and put
\begin{equation*}
\begin{split}
\widehat{\mathcal C}_n(\lambda):={}&
\mathbb E_{\theta\sim\widehat\rho_\lambda}
\sum_{j=1}^{M_n-1}L_{j,\theta}\\
&+\frac n\lambda\left\{
\KL(\widehat\rho_\lambda\Vert\mu)
+p_n(\lambda)+Q_n(\lambda)\right\}
+T_{n,u}
\end{split}
\end{equation*}
where 
\[
p_n(\lambda):=\log\left(\frac{K_n\lambda_+}{\lambda}\right) \text{, }
Q_n(\lambda):=
\sum_{\substack{1\leq j<M_n\\ \Delta_j\leq u_n}}
q\left(\frac{\lambda\Delta_j}{n}\right)
-\log\left(1-\frac{\lambda u_n}{n}\right),
\]
and \(T_{n,u}\) is defined in Corollary~\ref{cor:empirical-bernstein}.

On \(\{D_n>0\}\), dividing this criterion by \(D_n\) gives the empirical part of the PAC-Bayes bound, with the grid penalty included and the term involving \(\beta(u_n)\) omitted. The term \(T_{n,u}\) does not depend on \(\lambda\), so it can also be omitted when minimizing. Choose
\begin{equation}
\widehat\lambda\in
\arg\min_{\lambda\in\Lambda_n}\widehat{\mathcal C}_n(\lambda),
\label{eq:lambda-selection}
\end{equation}
breaking ties by taking the smallest minimizer (this definition also applies when \(D_n=0\), in which case our convention $R_n^{\mathrm{reg}}(\theta)=0$ implies  \(\hat{\rho}_\lambda \mu\) ).

\begin{theorem}
\label{thm:expectation-oracle}
For every finite deterministic grid \(\Lambda_n\subset(0,n/u_n)\) as above, if \(\mathbb E[D_n]>0\), the selected posterior satisfies
\begin{align}
&\mathbb E\left[
\mathbb E_{\theta\sim\widehat\rho_{\widehat\lambda}}R(\theta)
\right]
\notag\\
&\leq
\inf_{\lambda\in\Lambda_n}
\Bigg\{
\inf_{\rho\ll\mu}
\left[
\mathbb E_{\theta\sim\rho}R(\theta)
+\frac{2n}{\lambda\mathbb E[D_n]}\KL(\rho\Vert\mu)
\right]
+\frac{n\{p_n(\lambda)+2\mathbb E[Q_n(\lambda)]\}}
{\lambda\mathbb E[D_n]}
\Bigg\}
\notag\\
&\quad+
\frac{2\mathbb E[M_n]\beta(u_n)+2\mathbb E[T_{n,u}]}
{\mathbb E[D_n]}
+\frac{n}{\lambda_+\mathbb E[D_n]}
+\frac{n-\mathbb E[D_n]}n.
\label{eq:expectation-oracle}
\end{align}
\end{theorem}
The penalty \(p_n(\lambda)\) comes from a union bound over the grid; its form bounds the additional expectation remainder by \(n/(\lambda_+\mathbb E[D_n])\). The inequality
\[
\frac n\lambda M_n\log\left(1+\frac{\lambda\beta(u_n)}n\right)
\leq M_n\beta(u_n)
\]
provides a common upper bound for the omitted term, independently of \(\lambda\). Thus the selection rule requires no information about the regeneration tails.

A classical assumption in Bayesian statistics and PAC-Bayes bounds is the prior mass condition. Write
\[
R_*:=\inf_{\theta\in\Theta}R(\theta),
\qquad
V(\epsilon):=\{\theta\in\Theta:R(\theta)-R_*\leq\epsilon\},
\]
and assume that there are $c\in(0,1)$ and $d\geq 1$ such that, for every \(\epsilon\in(0,1)\),
\begin{equation}
\label{eq:priormass}
\mu(V(\epsilon))\geq c\epsilon^d.
\end{equation}
Let \(\mu_{V(\epsilon)}\) denote the prior conditioned on \(V(\epsilon)\). Then
\[
\mathbb E_{\theta\sim\mu_{V(\epsilon)}}R(\theta)\leq R_*+\epsilon,
\qquad
\KL(\mu_{V(\epsilon)}\Vert\mu)
=-\log\mu(V(\epsilon))
\leq\log(1/c)+d\log(1/\epsilon).
\]
While we an optimize in $\epsilon$, we can obtain an explicit rate with the simple choice \(\epsilon=1/n\). For \(n\geq2\), it yields
\begin{equation}
\inf_{\rho\ll\mu}
\left\{\mathbb E_{\theta\sim\rho}R(\theta)
+\frac{2n}{\lambda\mathbb E[D_n]}\KL(\rho\Vert\mu)\right\}
\leq R_*+\frac1n
+\frac{2n\{d\log n+\log(1/c)\}}
{\lambda\mathbb E[D_n]}.
\label{eq:prior-mass-rate}
\end{equation}
These choices of comparison posteriors are only used in the analysis; neither \(c\) nor \(d\) is needed to compute \(\widehat\lambda\).

We illustrate this result under exponential regeneration tails with unknown parameters.
\begin{corollary}
\label{cor:oracle:pacbayes}
In the setting of Theorem~\ref{thm:expectation-oracle}, assume \(\tau_1=0\), the prior mass condition \eqref{eq:priormass} with \(0<c\leq1\) and \(d\geq1\), and, for some unknown \(C_{\mathrm e},\kappa>0\), for any $t\geq 1$,
\[
\mathbb P_A(\Delta_1>t)\leq C_{\mathrm e}e^{-\kappa t},
.
\]
For \(n\geq2\), choose
\begin{equation}
u_n:=\left\lceil(\log n)^2\right\rceil,
\qquad
K_n:=1+\lceil\log_2 n\rceil,
\qquad
\Lambda_n:=
\left\{\frac{n}{2^{j+1}u_n}:j=0,\ldots,K_n-1\right\},
\label{eq:oracle-grid}
\end{equation}
and select \(\widehat\lambda\) by \eqref{eq:lambda-selection}. Thus \(\lambda_+=n/(2u_n)\) and \(p_n(n/(2^{j+1}u_n))=\log(K_n2^j)\). Put
\[
a_n:=2d\log n+2\log(1/c)+\log(2nK_n),
\qquad
v_2:=\mathbb E_A[\Delta_1^2].
\]
Then
\begin{align}
&\mathbb E\left[
\mathbb E_{\theta\sim\widehat\rho_{\widehat\lambda}}R(\theta)
\right] \leq R_*
+ \frac1n+
\frac{n}{\mathbb E[D_n]}
\Bigg[
3\sqrt{\frac{r_nv_2a_n}{n}}
+\frac{(2a_n+6)u_n}{n}
\notag\\
&\hspace{35mm}
+4r_nC_{\mathrm e}e^{-\kappa u_n}
\left(u_n+\frac1{1-e^{-\kappa}}\right)
\Bigg]
+\frac{C_{\mathrm e}e^{-\kappa}}
{n(1-e^{-\kappa})^2}.
\label{eq:oracle-exponential-adaptive}
\end{align}
For a fixed regenerative law and fixed \(c,d\), this gives
\begin{equation}
\mathbb E\left[
\mathbb E_{\theta\sim\widehat\rho_{\widehat\lambda}}R(\theta)
\right]-R_*
=O\left(\sqrt{\frac{d\log n}{n}}\right)
+O\left(\frac{d(\log n)^3}{n}\right).
\label{eq:oracle-exponential-adaptive-rate}
\end{equation}
The truncation level, the grid, and the selection criterion depend on neither \(c,d\) nor the unknown parameters of the regeneration tails.
\end{corollary}
Thus, we recover the same rate $O(\sqrt{\frac{d\log n}{n}})$ as in learning from i.i.d. data.
The $2^{j+1}$ in the construction of the grid was suggested by~\cite{cat2007}.
The universal choice \(u_n=\lceil(\log n)^2\rceil\) makes \(\beta(u_n)\) smaller than every fixed power of \(n^{-1}\) under exponential tails. The explicit terms after the square root in \eqref{eq:oracle-exponential-adaptive} are therefore second order for fixed \(c,d\). In particular, omitting \(\beta(u_n)\) from the criterion does not affect the leading rate.

\section{Minimax upper and lower bounds}
\label{sec:minimax}

In this section, we discuss the optimality of  Theorem~\ref{thm:regenerative-bernstein} and its corollaries. While the order $1/\sqrt{n}$ is expected, we will show that Theorem~\ref{thm:regenerative-bernstein} leads to an optimal dependence in the ratio $m/n$, where \(m:=\mathbb E_A[\Delta_1]\), under a moment assumption on the tail. Note that, in this optimality consideration, we are no longer concerned with empirical bounds: we will show that, even a statistician who would know $m$ and the constants in the moment assumptions would not be able to improve the order of magnitude of the bounds derived from Theorem~\ref{thm:regenerative-bernstein}.

First, we derive an upper bound on the estimation error from Theorem~\ref{thm:regenerative-bernstein} then we will compare it with a minimax lower bound. We state a proposition for a fixed tuning parameter; choosing that parameter gives the upper inequality in Theorem~\ref{thm:moment-minimax}.

Throughout this section, we still assume that \(\tau_1=0\), and suppose that, for some \(\varepsilon>0\) and \(C\geq1\),
\begin{equation}
\mathbb E_A[\Delta_1^{2+\varepsilon}]
\leq C m^{2+\varepsilon}.
\label{eq:moment-assumption}
\end{equation}
For the upper bound, choose the truncation level
\[
u_n=m\left(\frac nm\right)^{1/(2+\varepsilon)}.
\]

The moment condition controls both the quadratic term and the discarded tail. It also ensures that completed cycles cover at least half of the observation horizon on the event used in the next proposition.

\begin{proposition}
\label{prop:moment-confidence}
Fix \(\theta\in\Theta\), \(\delta\in(0,1)\), \(n\geq48Cm/\delta\), and a deterministic \(0<\lambda\leq n/(2u_n)\). There is an event of probability at least \(1-\delta\) on which \(D_n\geq n/2\) and
\begin{align}
\left|R(\theta)-\widehat R_n^{\mathrm{reg}}(\theta)\right|
\leq \frac2\lambda\log\frac8\delta
+\frac{24C\lambda m}{\delta n}
+\left(\frac{48C}{\delta}+4\right)
\left(\frac mn\right)^{(1+\varepsilon)/(2+\varepsilon)}.
\label{eq:moment-confidence}
\end{align}
\end{proposition}

Choosing \(\lambda\) balances the reciprocal and linear terms in \eqref{eq:moment-confidence}. We compare the resulting order with what any estimator can achieve. Consider the scalar model with \(\Theta=\{\theta_0\}\), \(X_t=0\), \(Y_t\in[0,1]\), and \(\ell_{\theta_0}(Z_t)=Y_t\). Let \(\mathfrak P(m,C,\varepsilon)\) denote the class of laws of i.i.d. regenerative blocks with positive integer lengths of mean \(m\) satisfying \eqref{eq:moment-assumption}, with observations starting at a regeneration time. Define the long-run risk under \(P\) by
\[
R_P(\theta_0)
:=\frac{\mathbb E_{A,P}\!\left[\sum_{t=1}^{\Delta_1}Y_t\right]}m.
\]
For \(\delta\in(0,1)\), define the minimax estimation error by
\begin{align*}
r_{n,\delta}(m,C,\varepsilon)
:=\inf_{\widetilde R_n}\inf\Bigl\{r\geq0:
\sup_{P\in\mathfrak P(m,C,\varepsilon)}
\mathbb P_{A,P}\bigl(
|\widetilde R_n-R_P(\theta_0)|>r
\bigr)\leq\delta\Bigr\},
\end{align*}
where the first infimum is over real-valued measurable estimators based on the observations and observed regeneration boundaries. For the upper bound, use the completed-cycle empirical risk, setting it to \(1/2\) when \(D_n=0\). Write
\[
C_{C,\delta}
:=8\sqrt{\frac{3C}{\delta}\log\frac8\delta}
+4\log\frac8\delta+\frac{48C}{\delta}+4.
\]

\begin{theorem}
\label{thm:moment-minimax}
\label{cor:moment-minimax}
Assume that \(m\in\mathbb N\). Then, for every \(\delta\in(0,3/8)\) and \(n\geq48Cm/\delta\),
\begin{equation}
\frac1{16\sqrt2}\sqrt{\frac mn}
\leq r_{n,\delta}(m,C,\varepsilon)
\leq C_{C,\delta}\sqrt{\frac mn}.
\label{eq:minimax-match}
\end{equation}
Thus, for each fixed \(\delta\in(0,3/8)\), the minimax estimation error is of order \(\sqrt{m/n}\).
\end{theorem}

\section{Self-normalized Gaussian approximation}
\label{sec:regenerative-gaussian}

The moment assumption \eqref{eq:moment-assumption} also allows us to estimate the variance of Gaussian fluctuations from the completed blocks. We use the self-normalized Bernstein inequality in Theorem~\ref{thm:regenerative-bernstein} to control the empirical centering in this estimator. We then combine this control with the regenerative strong approximation of \citet[Theorem~1]{bashtova_shashkin_heavy_2020}.

Keep \(m:=\mathbb E_A[\Delta_1]\), fix \(\theta\in\Theta\), and put \(p:=2+\varepsilon\). Define the completed-block and full-path centered sums by
\[
S_{n,\theta}^{\mathrm{reg}}
:=\sum_{j=1}^{M_n-1}\xi_{j,\theta},
\qquad
S_{n,\theta}^{\mathrm{full}}
:=\sum_{t=1}^n\{R(\theta)-\ell_\theta(Z_t)\}.
\]
Both sums are zero at \(n=0\). On \(\{D_n>0\}\), define the observable quadratic sum and variance estimator by
\begin{equation*}
\widehat Q_{n,\theta}
:=\sum_{j=1}^{M_n-1}
\left\{L_{j,\theta}-\Delta_j\widehat R_n^{\mathrm{reg}}(\theta)\right\}^2,
\qquad
\widehat\sigma_{n,\theta}^2:=\frac{\widehat Q_{n,\theta}}{D_n}.
\end{equation*}
Set both quantities to zero when \(D_n=0\).

\begin{theorem}
\label{thm:regenerative-gaussian}
Assume that \(\tau_1=0\), that the regenerative blocks are i.i.d. with positive integer lengths, and that \eqref{eq:moment-assumption} holds for some \(\varepsilon>0\) and \(C\geq1\). Keep the regenerative law and \(\theta\) fixed as \(n\to\infty\), and suppose that
\[
v_\theta^2:=\mathbb E_A[\xi_{1,\theta}^2]>0,
\qquad
\sigma_\theta^2:=\frac{v_\theta^2}{m}.
\]
There is a realization of the joint law of the losses and regeneration boundaries, together with a standard Brownian motion \(W_\theta\), such that, almost surely,
\begin{equation}
\max_{0\leq k\leq n}
\left|S_{k,\theta}^{\mathrm{reg}}-\sigma_\theta W_\theta(k)\right|
=o(n^{1/p}).
\label{eq:regenerative-strong-approximation}
\end{equation}
Moreover,
\begin{equation}
\widehat\sigma_{n,\theta}^2-\sigma_\theta^2
=
\begin{cases}
o(n^{2/p-1}),&2<p<4,\\
O\!\left(\sqrt{\log(n+2)/n}\right),&p\geq4.
\end{cases}
\label{eq:regenerative-variance-rate}
\end{equation}
In particular, \(D_n\) and \(\widehat Q_{n,\theta}\) are positive eventually. Then,
\begin{align}
\max_{0\leq k\leq n}
\left|
\frac{S_{k,\theta}^{\mathrm{reg}}}{\widehat\sigma_{n,\theta}\sqrt n}
-\frac{W_\theta(k)}{\sqrt n}
\right|
&=o(n^{1/p-1/2}),
\label{eq:regenerative-studentized-approximation}\\
\sup_{0\leq t\leq1}
\left|
\frac{S_{\lfloor nt\rfloor,\theta}^{\mathrm{reg}}}{\sqrt{\widehat Q_{n,\theta}}}
-\frac{W_\theta(nt)}{\sqrt n}
\right|
&=o(n^{1/p-1/2}),
\label{eq:regenerative-functional-approximation}
\end{align}
and conclusions similar to \eqref{eq:regenerative-strong-approximation},\eqref{eq:regenerative-studentized-approximation}, and \eqref{eq:regenerative-functional-approximation} hold for $S^{\mathrm{full}}$ as well.
Consequently,
\begin{equation}
\frac{\sqrt{D_n}\{R(\theta)-\widehat R_n^{\mathrm{reg}}(\theta)\}}
{\widehat\sigma_{n,\theta}}
=\frac{W_\theta(n)}{\sqrt n}+o(n^{1/p-1/2})
\ \xrightarrow{\ d\ }\ \mathcal N(0,1).
\label{eq:regenerative-studentized-clt}
\end{equation}
For both $H=\mathrm{full}$ and $H=\mathrm{reg}$, the processes \(t\mapsto S_{\lfloor nt\rfloor,\theta}^{H}/\sqrt{\widehat Q_{n,\theta}}\) converge in distribution to standard Brownian motion in \(D[0,1]\) with the usual Skorokhod topology.
\end{theorem}
% The statement concerns a fixed loss and law; a coupling uniform over \(\Theta\) or for a data-selected posterior requires additional assumptions.

\section{Proofs}
\label{sec:proofs-results}

\subsection{Proof of Theorem~\ref{thm:regenerative-bernstein}}

For deterministic \(s\in(0,1/u_n)\), define 
\[
\Phi_s(x):=sx-\log(1+sx),\qquad x>-1/s,
\]
and, for \(m\in\mathbb N_0\),
\[
\mathcal M_{m,\theta}(s)
:=
\exp\left\{
s\sum_{j=1}^mY_{j,\theta}
-\sum_{j=1}^m\Phi_s(Y_{j,\theta})
-m\log(1+s\beta(u_n))
\right\}
\]
where we remind that
$$Y_{j,\theta}:=\xi_{j,\theta}\mathbf 1_{\{\Delta_j\leq u_n\}} .$$

\begin{lemma}
\label{lem:stopped-exponential}
Fix \(\theta\in\Theta\) and deterministic \(s\in(0,1/u_n)\). Then \((\mathcal M_{m,\theta}(s))_{m\geq0}\) is a nonnegative supermartingale with respect to \((\mathcal F_m)_{m\geq0}\), with \(\mathcal M_{0,\theta}(s)=1\). In particular,
\begin{equation}
\mathbb E[\mathcal M_{M_n,\theta}(s)]\leq1.
\label{eq:donsker-prelimit}
\end{equation}
\end{lemma}

\begin{proof}[Proof of Lemma~\ref{lem:stopped-exponential}]
We first bound the mean introduced by truncation, then verify the supermartingale property and apply optional stopping at \(M_n\), the index of the first cycle whose end is after time \(n\). Since the untruncated reward is centered and bounded by the block length,
\begin{align}
\mathbb E_A[Y_{j,\theta}]
&=-\mathbb E_A\left[
\xi_{j,\theta}\mathbf 1_{\{\Delta_j>u_n\}}\right]
\notag\\
&\leq\mathbb E_A\left[
\Delta_j\mathbf 1_{\{\Delta_j>u_n\}}\right]
=\beta(u_n).
\label{eq:truncated-mean}
\end{align}
The admissible range of \(s\) ensures that
\[
|Y_{j,\theta}|\leq u_n,
\qquad 1+sY_{j,\theta}>0.
\]
Thus \(\Phi_s(Y_{j,\theta})\) is well defined. Recall from the definition of $\Phi$ that
\[
e^{sx-\Phi_s(x)}=1+sx.
\]
Independence of the marked cycles from the preceding block history, together with \eqref{eq:truncated-mean}, now gives
\[
\begin{split}
&\mathbb E\left[
\exp\left\{sY_{j,\theta}-\Phi_s(Y_{j,\theta})
-\log(1+s\beta(u_n))\right\}
\,\middle|\,\mathcal F_{j-1}\right]\\
&\qquad=\frac{1+s\mathbb E[Y_{j,\theta}\mid\mathcal F_{j-1}]}
{1+s\beta(u_n)}\leq1.
\end{split}
\]
Multiplying by the preceding value of \(\mathcal M_{j-1,\theta}(s)\) proves that the process is a nonnegative supermartingale starting at one.

It remains to check that optional stopping applies. The random index \(M_n\) is a stopping time because
\[
\{M_n\leq m\}=\{\tau_{m+1}>n\}\in\mathcal F_m.
\]
Positive integer cycle lengths also imply
\[
M_n\leq n+1.
\]
The bounded optional stopping theorem for supermartingales \cite[Chapter~10, Section~7, p.~495]{gut_probability_2005} therefore yields
\[
\mathbb E[\mathcal M_{M_n,\theta}(s)]
\leq\mathbb E[\mathcal M_{0,\theta}(s)]=1,
\]
which is \eqref{eq:donsker-prelimit}.
\end{proof}

We are now ready to prove the theorem.

\begin{proof}[Proof of Theorem~\ref{thm:regenerative-bernstein}]
We apply the stopped exponential bound to both signs of the truncated reward. We then remove the first cycle ending after time \(n\) and restore the rewards discarded by truncation. We work on \(\{D_n>0\}\), where the empirical risk is an average over completed cycles.

Set \(s:=\lambda/n\in(0,1/u_n)\). Replacing \(\ell_\theta\) by \(1-\ell_\theta\) changes the sign of its centered block reward. Lemma~\ref{lem:stopped-exponential}, Markov's inequality, and a union bound therefore give an event of probability at least \(1-\delta\) on which, for each \(a\in\{-1,1\}\),
\[
 s\sum_{j=1}^{M_n}aY_{j,\theta}
 \leq\sum_{j=1}^{M_n}\Phi_s(aY_{j,\theta})
 +M_n\log(1+s\beta(u_n))+\log\frac2\delta.
\]
For \(b\geq0\) with \(sb<1\), the definition of \(\Phi_s\) gives
\[
 \Phi_s(x)\leq q(sb),\qquad x\in[-b,b].
\]
Thus the completed cycles satisfy, for both signs,
\[
 \Phi_s(aY_{j,\theta})
 \leq\mathbf 1_{\{\Delta_j\leq u_n\}}q(s\Delta_j),
 \qquad j<M_n.
\]
For the additional cycle, truncation gives
\[
 |Y_{M_n,\theta}|\leq u_n,
 \qquad\Phi_s(aY_{M_n,\theta})\leq q(su_n).
\]
Removing its reward costs at most \(su_n\). Since
\[
 q(su_n)+su_n=-\log(1-su_n),
\]
the two signed inequalities imply
\[
 \begin{split}
 s\left|\sum_{j=1}^{M_n-1}Y_{j,\theta}\right|
 \leq{}&\log\frac2\delta+M_n\log(1+s\beta(u_n))\\
 &+\sum_{j=1}^{M_n-1}\mathbf 1_{\{\Delta_j\leq u_n\}}q(s\Delta_j)
 -\log(1-su_n).
 \end{split}
\]
To restore the untruncated rewards, use
\[
 \begin{split}
 D_n\left|R(\theta)-\widehat R_n^{\mathrm{reg}}(\theta)\right|
 &=\left|\sum_{j=1}^{M_n-1}\xi_{j,\theta}\right|\\
 &\leq\left|\sum_{j=1}^{M_n-1}Y_{j,\theta}\right|
 +\sum_{j=1}^{M_n-1}\Delta_j\mathbf 1_{\{\Delta_j>u_n\}}.
 \end{split}
\]
Substitute the bound for the truncated sum and divide by \(D_n\) to obtain \eqref{eq:regenerative-bernstein}. The convention at \(D_n=0\) covers the remaining case.
\end{proof}

\subsection{Proof of Corollary~\ref{cor:empirical-bernstein}}

\begin{proof}[Proof of Corollary~\ref{cor:empirical-bernstein}]
We first prove an inequality that holds simultaneously for finitely many deterministic values of the tuning parameter. We can then choose one of these values using the observed block lengths. The case \(D_n=0\) is covered by convention, so assume \(D_n>0\).

Consider the values
\[
s_j:=\frac{2^{-(j+1)}}{u_n},\qquad 0\leq j<\frac{J_n}{2}.
\]
For each of these values, apply Theorem~\ref{thm:regenerative-bernstein} with $\lambda=ns_j$, and replacing $\delta$ by $\frac{2\delta}{J_n}.$ There are \(J_n/2\) such values, so a union bound gives an event of probability at least \(1-\delta\) on which all the resulting inequalities hold. On this event, the elementary bounds
\[
q(x)\leq x^2,\qquad -\log(1-x)\leq2x
\quad(0\leq x\leq1/2),\qquad
\log(1+y)\leq y\quad(y\geq0)
\]
give, simultaneously for every \(s_j\),
\[
\begin{split}
D_n\left|R(\theta)-\widehat R_n^{\mathrm{reg}}(\theta)\right|
\leq{}&\frac{\log(J_n/\delta)}{s_j}+s_jH_{n,u}+M_n\beta(u_n)+2u_n+T_{n,u}.
\end{split}
\]

It remains to choose one of the values \(s_j\) so that the sum of the first two terms is close to its minimum. If \(s\) could be chosen freely subject to \(s\leq1/(2u_n)\), the minimizing value would be
\[
s_*:=\min\left\{\sqrt{\frac{\log(J_n/\delta)}{H_{n,u}}},
\frac1{2u_n}\right\}.
\]
When \(H_{n,u}=0\), take \(s_*=1/(2u_n)\). The smallest value in our collection is no greater than \(s_*\). To see this, note that
\[
H_{n,u}\leq D_n^2\leq n^2,
\qquad\log(J_n/\delta)\geq\log2.
\]
If \(n/u_n\geq1\), the smallest value consequently satisfies
\[
s_{J_n/2-1}\leq\frac1{2n}
<\frac{\sqrt{\log(J_n/\delta)}}n\leq s_*
\]
whenever the square-root term determines \(s_*\). If the minimum is attained at the largest allowed value, then \(s_*=s_0\). If \(n/u_n<1\), then \(J_n=2\), and the same bounds show that \(s_*=s_0\) also in this case.

When the square-root term is smaller than \(1/(2u_n)\), successive values \(s_j\) differ by a factor of two, so there is an index \(j\) such that
\[
\frac{s_*}{2}\leq s_j\leq s_*.
\]
Then
\[
\frac{\log(J_n/\delta)}{s_j}+s_jH_{n,u}
\leq3\sqrt{H_{n,u}\log(J_n/\delta)}.
\]
When \(s_*=1/(2u_n)\), choose \(s_0=1/(2u_n)\). In this case,
\[
\frac{H_{n,u}}{2u_n}\leq\sqrt{H_{n,u}\log(J_n/\delta)},
\]
and hence
\[
\frac{\log(J_n/\delta)}{s_0}+s_0H_{n,u}
\leq2u_n\log(J_n/\delta)
+\sqrt{H_{n,u}\log(J_n/\delta)}.
\]
Substituting either estimate into the inequality that holds for every \(s_j\), and dividing by \(D_n\), proves \eqref{eq:empirical-bernstein}.
\end{proof}

\subsection{Proof of Theorem~\ref{thm:expectation-deviation}}

\begin{proof}[Proof of Theorem~\ref{thm:expectation-deviation}]
We integrate the Bernstein tail bound before estimating its random remainder. This first controls the deviation weighted by \(D_n\); the length outside the completed cycles accounts for removing that weight. For the calculations below, write
\[
e_n:=\left|R(\theta)-\widehat R_n^{\mathrm{reg}}(\theta)\right|,
\qquad s:=\frac{\lambda}{n},
\]
and denote the logarithmic correction by
\[
\begin{split}
C_n(s):={}&M_n\log(1+s\beta(u_n))\\
&+\sum_{j=1}^{M_n-1}\mathbf 1_{\{\Delta_j\leq u_n\}}q(s\Delta_j)
-\log(1-su_n).
\end{split}
\]
We use \(H_{n,u}\) and \(T_{n,u}\) as defined in Corollary~\ref{cor:empirical-bernstein}. All quantities are integrable. The number of blocks is at most \(n+1\), the completed lengths and tail sum are at most \(n\), and \(su_n\leq1/2\).

\medskip
For every \(t>0\), apply Theorem~\ref{thm:regenerative-bernstein} with \(\delta=e^{-t}\). It gives
\[
\mathbb P\left(sD_ne_n-C_n(s)-sT_{n,u}>\log2+t\right)\leq e^{-t}.
\]
This statement also holds when \(D_n=0\). Since there are no completed blocks, the expression on the left of the threshold is \(-C_n(s)\leq0\). Integrating its positive part yields
\[
\mathbb E\left[(sD_ne_n-C_n(s)-sT_{n,u}-\log2)_+\right]
\leq\int_0^\infty e^{-t}\,dt=1.
\]
Therefore
\[
s\mathbb E[D_ne_n]
\leq1+\log2+\mathbb E[C_n(s)]+s\mathbb E[T_{n,u}].
\]
Since the loss and empirical risk lie in \([0,1]\),
\[
e_n=\frac{D_n}{n}e_n+\left(1-\frac{D_n}{n}\right)e_n
\leq\frac{D_n}{n}e_n+\frac{n-D_n}{n}.
\]
Taking expectations and using \(sn=\lambda\), we obtain
\begin{equation}
\mathbb E[e_n]
\leq\frac{1+\log2+\mathbb E[C_n(s)]}{\lambda}
+\frac{\mathbb E[T_{n,u}]}n+\frac{n-\mathbb E[D_n]}n.
\label{eq:expectation-deviation-integrated}
\end{equation}

\medskip
We bound the sum over completed cycles by including cycle \(M_n\), the first cycle whose end is after time \(n\). The enlarged sum ends at a stopping time, so we can use the following expectation identity. For any measurable \(f:\mathbb N\to[0,\infty)\) with \(\mathbb E_A[f(\Delta_1)]<\infty\), independence and \(\{M_n\geq j\}\in\mathcal F_{j-1}\) give
\[
\begin{split}
\mathbb E\left[\sum_{j=1}^{M_n}f(\Delta_j)\right]
&=\sum_{j=1}^{n+1}\mathbb E\left[
\mathbf 1_{\{M_n\geq j\}}f(\Delta_j)\right]\\
&=\sum_{j=1}^{n+1}\mathbb P(M_n\geq j)\,
\mathbb E_A[f(\Delta_1)]\\
&=\mathbb E[M_n]\mathbb E_A[f(\Delta_1)].
\end{split}
\]
Apply this identity to the truncated squared length and the discarded length. It follows that
\begin{align}
\mathbb E[H_{n,u}]&\leq\mathbb E[M_n]v(u_n),
\label{eq:wald-square-exponential}\\
\mathbb E[T_{n,u}]&\leq\mathbb E[M_n]\beta(u_n).
\label{eq:wald-tail-exponential}
\end{align}
For the logarithmic terms, the admissible range gives
\begin{align*}
q(x)&\leq x^2,\qquad 0\leq x\leq1/2,\\
-\log(1-x)&\leq2x,\qquad 0\leq x\leq1/2.
\end{align*}
Together with \(\log(1+x)\leq x\) for nonnegative \(x\), these bounds imply
\[
\begin{split}
\mathbb E[C_n(s)]
&\leq s\beta(u_n)\mathbb E[M_n]+s^2\mathbb E[H_{n,u}]+2su_n\\
&\leq s\beta(u_n)\mathbb E[M_n]
+s^2v(u_n)\mathbb E[M_n]+2su_n.
\end{split}
\]
Substituting this estimate and \eqref{eq:wald-tail-exponential} into \eqref{eq:expectation-deviation-integrated}, and replacing \(s\) by \(\lambda/n\), proves \eqref{eq:expectation-deviation}.
\end{proof}

\subsection{Proof of Corollary~\ref{cor:deviation-exponential} and Corollary~\ref{cor:deviation-polynomial}}

For both proofs, we will use the following notations:
\[
S_0:=0,\qquad S_k:=\sum_{j=1}^k\Delta_j\quad(k\in\mathbb N),\qquad
N_n:=\max\{k\geq0:S_k\leq n\}.
\]
Then \(M_n=N_n+1\), \(D_n=S_{N_n}\), and
\[
A_n:=n-D_n,\qquad
m:=\mathbb E_A[\Delta_1].
\]
Finally, recall that
$$
r_n=\frac{\mathbb E[M_n]}n.
$$
Since the lengths are positive integers, \(M_n\leq n+1\) and \(r_n\leq(n+1)/n\leq2\). For a fixed regenerative law, \(r_n\to1/m\) and \(\mathbb E[D_n]/n\to1\).

\begin{proof}[Proof of Corollary~\ref{cor:deviation-exponential}]
We first bound the length-dependent quantities in Theorem~\ref{thm:expectation-deviation}, then minimize the resulting bound in \(\lambda\).

For every integer \(u\geq1\), the tail-sum identity gives
\[
\begin{split}
\beta(u)
&=u\mathbb P_A(\Delta_1>u)+\sum_{k=u}^{\infty}\mathbb P_A(\Delta_1>k)\\
&\leq C_{\mathrm e}u e^{-\kappa u}
+C_{\mathrm e}\sum_{k=u}^{\infty}e^{-\kappa k}\\
&=C_{\mathrm e}e^{-\kappa u}
\left(u+\frac1{1-e^{-\kappa}}\right),
\end{split}
\]
which proves \eqref{eq:exponential-beta}. The prescribed truncation level satisfies
\[
e^{-\kappa u_n}\leq\frac1n,\qquad
1\leq u_n\leq1+\frac{\log n}{\kappa},
\]
and therefore gives both inequalities in \eqref{eq:exponential-beta-un}. The second-moment tail-sum identity similarly gives
\[
\begin{split}
v_2
&=1+\sum_{k=1}^{\infty}(2k+1)\mathbb P_A(\Delta_1>k)\\
&\leq1+C_{\mathrm e}\sum_{k=1}^{\infty}(2k+1)e^{-\kappa k}
=1+\frac{C_{\mathrm e}e^{-\kappa}(3-e^{-\kappa})}
{(1-e^{-\kappa})^2}.
\end{split}
\]
In particular, \(v(u_n)\leq v_2<\infty\). Theorem~\ref{thm:expectation-deviation} now gives
\begin{equation}
\begin{split}
\mathbb E\left[\left|R(\theta)-\widehat R_n^{\mathrm{reg}}(\theta)\right|\right]
\leq{}&\frac{1+\log2}{\lambda}+\frac{r_n\lambda v_2}n
+2r_n\beta(u_n)\\
&+\frac{2u_n}n+\frac{\mathbb E[A_n]}n.
\end{split}
\label{eq:exponential-before-age}
\end{equation}

To bound the renewal age, define
\[
U(s):=\sum_{k\geq0}\mathbb P(S_k=s),\qquad s\in\mathbb N_0.
\]
The renewal epochs are strictly increasing, so \(U(s)\leq1\). If \(A_n\geq t\), the last renewal occurs at some \(s\leq n-t\), and the following cycle ends after \(n\). Independence of that next cycle gives
\begin{equation}
\begin{split}
\mathbb P(A_n\geq t)
&=\sum_{s=0}^{n-t}U(s)\mathbb P_A(\Delta_1>n-s)\\
&\leq\sum_{r=t}^n\mathbb P_A(\Delta_1>r),
\qquad 1\leq t\leq n.
\end{split}
\label{eq:renewal-age-tail}
\end{equation}
Summing over \(t\) and changing the order of the finite sums yields
\begin{align}
\mathbb E[A_n]
&\leq\sum_{r=1}^n r\mathbb P_A(\Delta_1>r)
\notag\\
&\leq C_{\mathrm e}\sum_{r=1}^{\infty}re^{-\kappa r}
=\frac{C_{\mathrm e}e^{-\kappa}}{(1-e^{-\kappa})^2}.
\label{eq:exponential-age}
\end{align}
Substituting \eqref{eq:exponential-beta} and \eqref{eq:exponential-age} into \eqref{eq:exponential-before-age} proves \eqref{eq:exponential-unoptimized}.

For the optimization, let \(a,b,c>0\). Differentiating \(a/\lambda+b\lambda\) shows that its minimizer on \(0<\lambda\leq c\) is \(\lambda_*=\min\{\sqrt{a/b},c\}\), and
\begin{equation}
\min_{0<\lambda\leq c}\left(\frac a\lambda+b\lambda\right)
=
\begin{cases}
2\sqrt{ab},&\sqrt{a/b}\leq c,\\[1mm]
a/c+bc,&\sqrt{a/b}>c
\end{cases}
\ \leq 2\sqrt{ab}+\frac ac.
\label{eq:deviation-lambda-optimization}
\end{equation}
Indeed, in the second case \(bc<\sqrt{ab}\). Taking
\[
a=1+\log2,\qquad b=\frac{r_nv_2}{n},\qquad c=\frac n{2u_n}
\]
proves that \eqref{eq:exponential-lambda} is the exact minimizer of \eqref{eq:exponential-unoptimized}. Moreover,
\[
\frac{1+\log2}{\lambda_n}+\frac{r_n\lambda_n v_2}{n}
\leq2\sqrt{\frac{(1+\log2)r_nv_2}{n}}
+\frac{2(1+\log2)u_n}{n}.
\]
Combining this inequality with \(e^{-\kappa u_n}\leq1/n\) proves \eqref{eq:exponential-order-bound}.

Finally, we record the renewal limits, which also apply under the polynomial assumption because they require only finite mean cycle length. The strong law gives \(S_k/k\to m\) almost surely. Since
\[
S_{N_n}\leq n<S_{N_n+1}
\]
and \(N_n\to\infty\), it follows that
\[
\frac{N_n}{n}\longrightarrow\frac1m,\qquad
\frac{M_n}{n}\longrightarrow\frac1m,\qquad
\frac{D_n}{n}\longrightarrow1
\qquad\text{almost surely}.
\]
The bounds \(M_n/n\leq2\) and \(D_n/n\leq1\) permit dominated convergence, so
\[
r_n\longrightarrow\frac1m,\qquad
\frac{\mathbb E[D_n]}n\longrightarrow1.
\]
As \(u_n=O(\log n)\), the constraint on \(\lambda_n\) is eventually inactive and \(\lambda_n\sim\sqrt{(1+\log2)mn/v_2}\). In \eqref{eq:exponential-order-bound}, the square-root term has the leading constant stated in \eqref{eq:exponential-unoptimized-asymptotic}, and all the other terms are \(O((\log n)/n)\).
\end{proof}

\begin{proof}[Proof of Corollary~\ref{cor:deviation-polynomial}]
For every real \(u\geq1\), the tail-integral identity gives
\[
\begin{split}
\beta(u)
&=u\mathbb P_A(\Delta_1>u)
+\int_u^\infty\mathbb P_A(\Delta_1>t)\,dt\\
&\leq C_{\mathrm p}u^{1-\alpha}
+C_{\mathrm p}\int_u^\infty t^{-\alpha}\,dt
=\frac{\alpha C_{\mathrm p}}{\alpha-1}u^{1-\alpha}.
\end{split}
\]
This proves \eqref{eq:polynomial-beta}. Since \(u_n\geq n^{1/\alpha}\) and \(1-\alpha<0\), it also proves \eqref{eq:polynomial-beta-un}.

For the truncated second moment, Tonelli's theorem yields
\[
\begin{split}
v(u)&\leq\mathbb E_A[(\Delta_1\wedge u)^2]
=2\int_0^u t\mathbb P_A(\Delta_1>t)\,dt\\
&\leq1+2C_{\mathrm p}\int_1^u t^{1-\alpha}\,dt.
\end{split}
\]
For \(\alpha>2\), the last integral is at most \(1/(\alpha-2)\); for \(\alpha=2\), it is \(\log u\); and for \(1<\alpha<2\), it is at most \(u^{2-\alpha}/(2-\alpha)\). Using \(u^{2-\alpha}\geq1\) in the last case proves
\begin{equation*}
v(u)\leq V_\alpha(u).
\end{equation*}
Consequently, Theorem~\ref{thm:expectation-deviation} gives
\begin{equation}
\begin{split}
\mathbb E\left[\left|R(\theta)-\widehat R_n^{\mathrm{reg}}(\theta)\right|\right]
\leq{}&\frac{1+\log2}{\lambda}+\frac{r_n\lambda}nV_\alpha(u_n)
+2r_n\beta(u_n)\\
&+\frac{2u_n}n+\frac{\mathbb E[A_n]}n.
\end{split}
\label{eq:polynomial-master}
\end{equation}

The renewal-age argument in \eqref{eq:renewal-age-tail} applies without any exponential tail assumption. Thus
\[
\mathbb E[A_n]\leq\sum_{r=1}^n r\mathbb P_A(\Delta_1>r)
\leq C_{\mathrm p}\sum_{r=1}^n r^{1-\alpha}.
\]
The decreasing summand can be bounded by an integral:
\[
\sum_{r=1}^n r^{1-\alpha}\leq
\begin{cases}
\displaystyle 1+\int_1^\infty t^{1-\alpha}\,dt
=1+\frac1{\alpha-2},&\alpha>2,\\[2mm]
\displaystyle 1+\log n,&\alpha=2,\\[2mm]
\displaystyle 1+\frac{n^{2-\alpha}}{2-\alpha},&1<\alpha<2.
\end{cases}
\]
It follows that
\begin{equation*}
\mathbb E[A_n]\leq G_\alpha(n).
\end{equation*}
Substitution into \eqref{eq:polynomial-master} proves \eqref{eq:polynomial-explicit-bound}.

Apply \eqref{eq:deviation-lambda-optimization} with
\[
a=1+\log2,\qquad b=\frac{r_nV_\alpha(u_n)}n,\qquad
c=\frac n{2u_n}.
\]
This proves that \eqref{eq:polynomial-lambda} minimizes \eqref{eq:polynomial-explicit-bound}, and gives
\[
\frac{1+\log2}{\lambda_n}
+\frac{r_n\lambda_n V_\alpha(u_n)}n
\leq2\sqrt{\frac{(1+\log2)r_nV_\alpha(u_n)}n}
+\frac{2(1+\log2)u_n}n.
\]
Adding the remaining terms proves \eqref{eq:polynomial-deviation-optimized}.

For the rates, the truncation level satisfies
\[
n^{1/\alpha}\leq u_n\leq2n^{1/\alpha},\qquad
u_n^{1-\alpha}\leq n^{-(\alpha-1)/\alpha},\qquad
\frac{u_n}n\leq2n^{-(\alpha-1)/\alpha}.
\]
Since \(r_n\leq2\), the definitions of \(V_\alpha\) and \(G_\alpha\) give
\[
\frac{V_\alpha(u_n)}n=
\begin{cases}
O(n^{-1}),&\alpha>2,\\
O((\log n)/n),&\alpha=2,\\
O\bigl(n^{-2(\alpha-1)/\alpha}\bigr),&1<\alpha<2,
\end{cases}
\]
and
\[
\frac{G_\alpha(n)}n=
\begin{cases}
O(n^{-1}),&\alpha>2,\\
O((\log n)/n),&\alpha=2,\\
O(n^{1-\alpha}),&1<\alpha<2.
\end{cases}
\]
When \(\alpha>2\), the square-root term in \eqref{eq:polynomial-deviation-optimized} is \(O(n^{-1/2})\), and the other terms are smaller because \((\alpha-1)/\alpha>1/2\). When \(\alpha=2\), the square-root term is \(O(\sqrt{(\log n)/n})\), the tail and truncation terms are \(O(n^{-1/2})\), and the age term is \(O((\log n)/n)\). When \(1<\alpha<2\), the square-root, tail, and truncation terms are all \(O(n^{-(\alpha-1)/\alpha})\), while
\[
n^{1-\alpha}=o\bigl(n^{-(\alpha-1)/\alpha}\bigr).
\]
This proves \eqref{eq:polynomial-alpha-greater-two}--\eqref{eq:polynomial-alpha-less-two}.

For the remarks, if \(\mathbb P_A(\Delta_1>t)\sim c\,t^{-\alpha}\), the same tail-integral identity gives
\[
\beta(u)\sim\frac{\alpha c}{\alpha-1}u^{1-\alpha}.
\]
Together with \(u_n\sim n^{1/\alpha}\) and \(r_n\to1/m\), this proves that the tail contribution can have leading order when \(1<\alpha<2\). At \(\alpha=2\), if \(v_2<\infty\), replacing \(V_2(u_n)\) by \(v_2\) makes the square-root term \(O(n^{-1/2})\), and all remaining terms are of this order or smaller. Moreover,
\[
u\beta(u)\leq
\mathbb E_A[\Delta_1^2\mathbf1_{\{\Delta_1>u\}}]\longrightarrow0,
\]
by integrability of \(\Delta_1^2\). Hence \(\beta(u_n)=o(n^{-1/2})\) for \(u_n=\lceil\sqrt n\rceil\), as asserted.
\end{proof}

\subsection{Proof of Corollary~\ref{cor:deviation-queue}}

\begin{proof}
We verify non-Markovianity, regeneration, and the cycle-length bounds, then apply Corollary~\ref{cor:deviation-exponential}.

The histories
\[
 (Q_0,Q_1,Q_2)=(0,0,1),
 \qquad (Q_0,Q_1,Q_2)=(0,1,1)
\]
have probabilities \((1-p)p\) and \(p^2\), respectively. In the first history, the customer has one service slot remaining. In the second, the first customer has departed and the remaining customer's service has not started. Consequently,
\[
 \begin{aligned}
 \mathbb P(Q_3=0\mid Q_0=0,Q_1=0,Q_2=1)&=1-p,\\
 \mathbb P(Q_3=0\mid Q_0=0,Q_1=1,Q_2=1)&=0.
 \end{aligned}
\]
Both histories end at \(Q_2=1\), so \((Q_t)\) is not Markov with respect to its natural filtration.

Each \(\tau_j\) is a stopping time for \(\sigma(Q_0,\ldots,Q_t)\) and for \(\mathcal G_t:=\sigma(I_0,\ldots,I_{t-1})\). At an empty epoch, the queue's future is determined by the unused arrivals \((I_{\tau_j+r})_{r\in\mathbb N_0}\), which are independent of \(\mathcal G_{\tau_j}\). The same service rule applies at every empty state, so the marked cycles are i.i.d. once their almost sure finiteness is established.

If \(I_0=0\), then \(\Delta_1=1\). Conditional on \(I_0=1\), let \(T\in\mathbb N\cup\{\infty\}\) be the number of customers served before the queue next becomes empty. The cycle lasts \(2T\) slots. Define the potential offspring counts by
\[
 X_1:=I_1,\qquad X_i:=I_{2i-2}+I_{2i-1},\qquad i\geq2.
\]
While the queue is busy, these count the new arrivals during each successive service. The first-slot arrival is the initial customer itself. Conditional on \(I_0=1\), these counts are independent, with \(X_1\sim\operatorname{Bernoulli}(p)\) and \(X_i\sim\operatorname{Binomial}(2,p)\) for \(i\geq2\). The unused arrivals define the counts even after the busy period ends.

For an independent \(V\sim\operatorname{Bernoulli}(p)\), put \(X'_1:=X_1+V\) and \(X'_i:=X_i\) for \(i\geq2\). The resulting counts are i.i.d. Binomial\((2,p)\). Their Galton--Watson total progeny \(T'\) dominates \(T\), and is finite almost surely because the offspring mean is \(2p<1\). If \(T'\geq k\), the first \(k-1\) customers have at least \(k-1\) offspring. Therefore, for each integer \(k\geq1\), the binomial Chernoff bound gives
\begin{align*}
 \mathbb P(T\geq k\mid I_0=1)
 &\leq\mathbb P(T'\geq k)\notag\\
 &\leq\mathbb P\bigl(\operatorname{Binomial}(2(k-1),p)\geq k-1\bigr)
 \leq e^{-2(k-1)\kappa_0}.
\end{align*}
For \(k=1\), the last bound is simply one. Taking \(k=\lfloor t/2\rfloor+1\) yields
\[
 \mathbb P(2T>t\mid I_0=1)\leq e^{2\kappa_0}e^{-\kappa_0t},\qquad t\geq0.
\]
For \(t\geq1\), only \(I_0=1\) contributes to \(\mathbb P(\Delta_1>t)\). For \(0\leq t<1\), \(\mathbb P(\Delta_1>t)=1\leq e^{2\kappa_0}e^{-\kappa_0t}\). This proves the stated exponential tail, finiteness of every cycle-length moment, and almost sure finiteness of all regeneration times.

To calculate the first two moments, let \(d:=1-2p\), and let \(H\) be the total progeny of a Galton--Watson process with offspring \(X\sim\operatorname{Binomial}(2,p)\). Its branching recursion
\[
H\ \overset{d}=\ 1+\sum_{i=1}^{X}H_i,
\]
where the \(H_i\) are independent copies of \(H\), independent of \(X\), gives
\[
\mathbb E[H]=1+2p\,\mathbb E[H]=\frac1d.
\]
Taking second moments in the same recursion and using \(\mathbb E[X(X-1)]=2p^2\), we obtain
\[
d\,\mathbb E[H^2]
=1+4p\,\mathbb E[H]+2p^2(\mathbb E[H])^2,
\qquad
\mathbb E[H^2]=\frac{1-2p^2}{d^3}.
\]
The initial customer has only \(J\sim\operatorname{Bernoulli}(p)\) offspring, so, conditional on \(I_0=1\), \(T\) has the distribution of \(1+JH\), with \(J\) and \(H\) independent. Hence
\[
\mathbb E[T\mid I_0=1]=1+\frac p d=\frac{1-p}d,
\]
and
\[
\mathbb E[T^2\mid I_0=1]
=1+\frac{2p}d+\frac{p(1-2p^2)}{d^3}
=\frac{(1-p)(1-2p+2p^2)}{d^3}.
\]
Since \(\Delta_1=1\) on \(\{I_0=0\}\) and \(\Delta_1=2T\) on \(\{I_0=1\}\),
\[
\begin{split}
\mathbb E[\Delta_1]
&=(1-p)+2p\,\mathbb E[T\mid I_0=1],\\
\mathbb E[\Delta_1^2]
&=(1-p)+4p\,\mathbb E[T^2\mid I_0=1].
\end{split}
\]
Substitution and simplification prove \eqref{eq:queue-cycle-moments}.

All the assumptions of Corollary~\ref{cor:deviation-exponential} now hold with \(C_{\mathrm e}=C_0\), \(\kappa=\kappa_0\), and \(\tau_1=0\). Its tail bounds give \eqref{eq:queue-beta} and \eqref{eq:queue-beta-un}, and its expected deviation and optimization statements give \eqref{eq:queue-deviation-unoptimized}--\eqref{eq:queue-deviation-optimized}. Finally, \(r_n\to1/m\) and
\[
\frac{v_2}{m}=\frac{1-2p+4p^2}{(1-2p)^2},
\]
so its asymptotic bound becomes \eqref{eq:queue-deviation-rate}.
\end{proof}

\subsection{Proof of Theorem~\ref{prop:initial-PACBayes}}

\begin{lemma}
\label{lem:stopped-pac-bayes}
For every deterministic \(s\in(0,1/u_n)\) and every \(\delta\in(0,1)\), with probability at least \(1-\delta\), simultaneously for every \(\rho\ll\mu\),
\begin{align}
s\mathbb E_{\theta\sim\rho}
\sum_{j=1}^{M_n}Y_{j,\theta}
\leq {}&
\mathbb E_{\theta\sim\rho}
\sum_{j=1}^{M_n}\Phi_s(Y_{j,\theta})
\notag\\
&+
M_n\log(1+s\beta(u_n))
+\KL(\rho\Vert\mu)
+\log\left(\frac{1}{\delta}\right).
\label{eq:stopped-pac-bayes}
\end{align}
\end{lemma}

\begin{proof}[Proof of Lemma~\ref{lem:stopped-pac-bayes}]
We average the stopped exponential bound under the prior. Markov's inequality then gives a common probability event, and the Donsker--Varadhan formula extends the bound to every posterior. This is the standard PAC--Bayes argument; see the proof of Theorem~2.1 in \cite{alquier_user-friendly_2024}.

Lemma~\ref{lem:stopped-exponential} and Tonelli's theorem give
\[
\mathbb E\left[\int_\Theta
\mathcal M_{M_n,\theta}(s)\,\mu(d\theta)\right]\leq1.
\]
Hence, with probability at least \(1-\delta\),
\[
\int_\Theta\mathcal M_{M_n,\theta}(s)\,\mu(d\theta)
\leq\delta^{-1}.
\]
On this event, the Donsker--Varadhan formula gives, simultaneously for every \(\rho\ll\mu\),
\begin{align*}
\mathbb E_{\theta\sim\rho}\log\mathcal M_{M_n,\theta}(s)
&\leq\KL(\rho\Vert\mu)
+\log\int_\Theta\mathcal M_{M_n,\theta}(s)\,\mu(d\theta)\\
&\leq\KL(\rho\Vert\mu)+\log(1/\delta).
\end{align*}
To identify the desired inequality, expand the logarithm:
\[
\begin{split}
\mathbb E_{\theta\sim\rho}\log\mathcal M_{M_n,\theta}(s)
={}&s\mathbb E_{\theta\sim\rho}\sum_{j=1}^{M_n}Y_{j,\theta}
-\mathbb E_{\theta\sim\rho}\sum_{j=1}^{M_n}\Phi_s(Y_{j,\theta})\\
&-M_n\log(1+s\beta(u_n)).
\end{split}
\]
Substituting this identity in the preceding bound and rearranging proves \eqref{eq:stopped-pac-bayes}.
\end{proof}

\begin{proof}[Theorem~\ref{prop:initial-PACBayes}]
We apply the stopped PAC--Bayes inequality, restrict its sum to the cycles completed by time \(n\), and restore the losses excluded by truncation. The case \(D_n=0\) is covered by convention. Throughout the proof, set
\[
s:=\frac{\lambda}{n},\qquad 0<s<\frac1{u_n}.
\]

\medskip
Assume that the event given in Lemma~\ref{lem:stopped-pac-bayes}, which has probability at least \(1-\delta\), holds. Its sum includes cycle \(M_n\), the first cycle whose end is after time \(n\). We bound its contribution before restricting the sum to the completed cycles. The correction for each completed block satisfies
\[
\Phi_s(Y_{j,\theta})
\leq\mathbf 1_{\{\Delta_j\leq u_n\}}q(s\Delta_j),
\qquad j<M_n.
\]
For cycle \(M_n\), truncation gives
\[
|Y_{M_n,\theta}|\leq u_n,
\qquad
\Phi_s(Y_{M_n,\theta})\leq q(su_n).
\]
In particular, removing its reward costs at most \(u_n\):
\begin{equation*}
\sum_{j=1}^{M_n-1}Y_{j,\theta}
=\sum_{j=1}^{M_n}Y_{j,\theta}-Y_{M_n,\theta}
\leq\sum_{j=1}^{M_n}Y_{j,\theta}+u_n.
\end{equation*}
Subtracting this cycle's reward and bounding its correction therefore costs at most
\[
u_n+\frac{q(su_n)}s=-\frac{\log(1-su_n)}s.
\]
Combining these estimates with \eqref{eq:stopped-pac-bayes} yields, simultaneously for every \(\rho\ll\mu\),
\begin{align}
\mathbb E_{\theta\sim\rho}\sum_{j=1}^{M_n-1}Y_{j,\theta}
\leq\frac1s\Bigg[
&\KL(\rho\Vert\mu)+\log\frac1\delta
+M_n\log(1+s\beta(u_n))\notag\\
&+\sum_{j=1}^{M_n-1}\mathbf 1_{\{\Delta_j\leq u_n\}}q(s\Delta_j)
-\log(1-su_n)\Bigg].
\label{eq:completed-truncated-sum}
\end{align}

\medskip
Since \(|\xi_{j,\theta}|\leq\Delta_j\),
\[
\begin{split}
\sum_{j=1}^{M_n-1}\xi_{j,\theta}
&=\sum_{j=1}^{M_n-1}Y_{j,\theta}
+\sum_{j=1}^{M_n-1}\xi_{j,\theta}\mathbf 1_{\{\Delta_j>u_n\}}\\
&\leq\sum_{j=1}^{M_n-1}Y_{j,\theta}
+\sum_{j=1}^{M_n-1}\Delta_j\mathbf 1_{\{\Delta_j>u_n\}}.
\end{split}
\]
The completed discrepancies equal the empirical risk error multiplied by the observed length:
\[
\sum_{j=1}^{M_n-1}\xi_{j,\theta}
=D_n\{R(\theta)-\widehat R_n^{\mathrm{reg}}(\theta)\}.
\]
Average over \(\rho\), substitute \eqref{eq:completed-truncated-sum}, and divide by \(D_n\). Recalling \(s=\lambda/n\) gives \eqref{eq:main-bound} on the same event, simultaneously for all \(\rho\ll\mu\).
\end{proof}

\subsection{Proof of Theorem~\ref{thm:expectation-oracle}}

\begin{proof}[Proof of Theorem~\ref{thm:expectation-oracle}]
We first control the forward deviation simultaneously over the grid with a common integrable remainder. This allows us to select \(\lambda\) from the data before comparing with a fixed posterior and a fixed grid point.

For \(\lambda\in\Lambda_n\), set \(s=\lambda/n\), and write
\[
C_n(s):=M_n\log(1+s\beta(u_n))+Q_n(\lambda),
\qquad
Z_\lambda:=
\int_\Theta\mathcal M_{M_n,\theta}(s)\,\mu(d\theta).
\]
Lemma~\ref{lem:stopped-exponential} and Tonelli's theorem imply \(\mathbb E[Z_\lambda]\leq1\). As in the proof of Theorem~\ref{prop:initial-PACBayes}, removing the cycle ending after time \(n\) gives, pathwise for every \(\theta\),
\begin{equation}
s\sum_{j=1}^{M_n-1}Y_{j,\theta}-C_n(s)
\leq\log\mathcal M_{M_n,\theta}(s).
\label{eq:cross-block-bound2}
\end{equation}
Indeed, the completed-cycle corrections are bounded by the sum in \(Q_n(\lambda)\), while
\[
sY_{M_n,\theta}-\Phi_s(Y_{M_n,\theta})
\geq-su_n-q(su_n)=\log(1-su_n).
\]
The Donsker--Varadhan formula, followed by restoration of the truncated rewards, therefore gives, for every \(\rho\ll\mu\) with finite KL divergence,
\begin{align*}
D_n\mathbb E_{\theta\sim\rho}R(\theta)
\leq{}&
\mathbb E_{\theta\sim\rho}\sum_{j=1}^{M_n-1}L_{j,\theta}
+\frac n\lambda
\left\{\KL(\rho\Vert\mu)+Q_n(\lambda)+\log Z_\lambda\right\}\\
&+M_n\beta(u_n)+T_{n,u}.
\end{align*}
Here we used \(\log(1+x)\leq x\) to bound the population tail term independently of \(\lambda\).

Define \(x_+:=\max\{x,0\}\) and
\[
W_n:=\max_{\lambda\in\Lambda_n}
\frac n\lambda\{\log Z_\lambda-p_n(\lambda)\}_+.
\]
Markov's inequality gives, for \(t\geq0\),
\[
\mathbb P\{\log Z_\lambda>p_n(\lambda)+t\}
\leq e^{-p_n(\lambda)-t}.
\]
Integrating this inequality and bounding the maximum by a sum yields
\begin{equation}
\mathbb E[W_n]
\leq\sum_{\lambda\in\Lambda_n}
\frac n\lambda e^{-p_n(\lambda)}
=\frac n{\lambda_+}.
\label{eq:grid-expectation-remainder}
\end{equation}
Consequently, simultaneously for every grid point and posterior,
\begin{align}
D_n\mathbb E_{\theta\sim\rho}R(\theta)
\leq{}&
\mathbb E_{\theta\sim\rho}\sum_{j=1}^{M_n-1}L_{j,\theta}
+\frac n\lambda
\left\{\KL(\rho\Vert\mu)+p_n(\lambda)+Q_n(\lambda)\right\}
\notag\\
&+M_n\beta(u_n)+T_{n,u}+W_n.
\label{eq:expectation-forward}
\end{align}
The quantities used below are integrable for the Gibbs posteriors and for each fixed comparator with finite KL divergence: \(M_n\leq n+1\), \(\lambda u_n/n<1\), and comparison with \(\mu\) in the Gibbs optimization gives
\[
0\leq\KL(\widehat\rho_\lambda\Vert\mu)
\leq\frac{\lambda D_n}{n}\leq\lambda.
\]
When \(D_n=0\), the Gibbs posterior is \(\mu\).

Apply \eqref{eq:expectation-forward} to \(\lambda=\widehat\lambda\) and \(\rho=\widehat\rho_{\widehat\lambda}\). By \eqref{eq:lambda-selection},
\[
D_n\mathbb E_{\theta\sim\widehat\rho_{\widehat\lambda}}R(\theta)
\leq\widehat{\mathcal C}_n(\widehat\lambda)
+M_n\beta(u_n)+W_n
\leq\widehat{\mathcal C}_n(\lambda)
+M_n\beta(u_n)+W_n
\]
for every fixed \(\lambda\in\Lambda_n\). The Gibbs variational identity further gives, for every fixed \(\rho\ll\mu\) with finite KL divergence,
\begin{align}
D_n\mathbb E_{\theta\sim\widehat\rho_{\widehat\lambda}}R(\theta)
\leq{}&
\mathbb E_{\theta\sim\rho}\sum_{j=1}^{M_n-1}L_{j,\theta}
+\frac n\lambda
\left\{\KL(\rho\Vert\mu)+p_n(\lambda)+Q_n(\lambda)\right\}
\notag\\
&+M_n\beta(u_n)+T_{n,u}+W_n.
\label{eq:expectation-after-optimality}
\end{align}
This argument uses no division by \(D_n\), so it also covers \(D_n=0\).

For the reverse comparison, apply \eqref{eq:cross-block-bound2} and the Donsker--Varadhan formula to the complementary loss \(1-\ell_\theta\), keeping \(\lambda\) and \(\rho\) fixed. The corresponding exponential integral \(Z_\lambda^{-}\) satisfies \(\mathbb E[Z_\lambda^{-}]\leq1\), and hence \(\mathbb E[\log Z_\lambda^{-}]\leq0\) by Jensen's inequality. The corrections \(Q_n(\lambda)\), \(M_n\beta(u_n)\), and \(T_{n,u}\) are unchanged. It follows that
\begin{equation}
\begin{split}
\mathbb E\left[
\mathbb E_{\theta\sim\rho}\sum_{j=1}^{M_n-1}L_{j,\theta}\right]
\leq{}&\mathbb E[D_n]\mathbb E_{\theta\sim\rho}R(\theta)\\
&+\frac n\lambda\left\{\KL(\rho\Vert\mu)
+\mathbb E[Q_n(\lambda)]\right\}
+\mathbb E[M_n]\beta(u_n)+\mathbb E[T_{n,u}].
\end{split}
\label{eq:expectation-reverse}
\end{equation}
Taking expectations in \eqref{eq:expectation-after-optimality} and using \eqref{eq:grid-expectation-remainder} and \eqref{eq:expectation-reverse} gives
\begin{align}
&\mathbb E\left[
D_n\mathbb E_{\theta\sim\widehat\rho_{\widehat\lambda}}R(\theta)\right]
\notag\\
&\leq\mathbb E[D_n]\mathbb E_{\theta\sim\rho}R(\theta)
+\frac n\lambda
\left\{2\KL(\rho\Vert\mu)+p_n(\lambda)
+2\mathbb E[Q_n(\lambda)]\right\}
\notag\\
&\quad+2\mathbb E[M_n]\beta(u_n)
+2\mathbb E[T_{n,u}]+\frac n{\lambda_+}.
\label{eq:expectation-weighted}
\end{align}

Finally, put \(X_n=\mathbb E_{\theta\sim\widehat\rho_{\widehat\lambda}}R(\theta)\). Since \(0\leq X_n\leq1\) and \(0\leq D_n\leq n\),
\[
\mathbb E[X_n]
=\frac1n\mathbb E[D_nX_n]
+\mathbb E\left[\left(1-\frac{D_n}n\right)X_n\right]
\leq\frac{\mathbb E[D_nX_n]}{\mathbb E[D_n]}
+\frac{n-\mathbb E[D_n]}n.
\]
Substitute \eqref{eq:expectation-weighted} and take the infimum over fixed \(\rho\ll\mu\) and \(\lambda\in\Lambda_n\). This proves \eqref{eq:expectation-oracle}.
\end{proof}

\subsection{Proof of Corollary~\ref{cor:oracle:pacbayes}}

\begin{proof}[Proof of Corollary~\ref{cor:oracle:pacbayes}]
Every grid point satisfies \(\lambda u_n/n\leq1/2\), so
\[
Q_n(\lambda)
\leq\frac{\lambda^2}{n^2}H_{n,u}
+\frac{2\lambda u_n}{n},
\]
using \(q(x)\leq x^2\) and \(-\log(1-x)\leq2x\) on \([0,1/2]\). The stopped-sum bounds \eqref{eq:wald-square-exponential} and \eqref{eq:wald-tail-exponential} give
\[
\mathbb E[H_{n,u}]\leq\mathbb E[M_n]v(u_n),
\qquad
\mathbb E[T_{n,u}]\leq\mathbb E[M_n]\beta(u_n).
\]
Moreover, for the grid in \eqref{eq:oracle-grid},
\[
p_n(\lambda)\leq\log K_n+(K_n-1)\log2
\leq\log(2nK_n),
\qquad
\frac1{\lambda_+}=\frac{2u_n}{n}.
\]
Apply Theorem~\ref{thm:expectation-oracle} and \eqref{eq:prior-mass-rate}. With
\[
b_n:=\frac{2r_nv(u_n)}n,
\]
we obtain
\begin{align}
&\mathbb E\left[
\mathbb E_{\theta\sim\widehat\rho_{\widehat\lambda}}R(\theta)\right]-R_*
\notag\\
&\leq\frac1n+
\frac n{\mathbb E[D_n]}
\left[
\min_{\lambda\in\Lambda_n}
\left\{\frac{a_n}{\lambda}+b_n\lambda\right\}
+\frac{6u_n}{n}+4r_n\beta(u_n)
\right]
+\frac{n-\mathbb E[D_n]}n.
\label{eq:adaptive-oracle-before-grid}
\end{align}
The term \(6u_n/n\) consists of \(4u_n/n\) from the two crossing-cycle corrections and \(2u_n/n\) from the grid expectation remainder.

We next bound the minimum over the grid. If \(b_n=0\), choose \(\lambda=\lambda_+\). Otherwise, the unconstrained minimizer is \(\lambda_*=\sqrt{a_n/b_n}\). Since \(r_n\leq2\), \(v(u_n)\leq u_n^2\), and \(a_n>1\),
\[
\lambda_*
\geq\frac{\sqrt{a_n n}}{2u_n}
\geq\frac1{2u_n}
\geq\min\Lambda_n.
\]
If \(\lambda_*\leq\lambda_+\), choose the grid point nearest to \(\lambda_*\) on the logarithmic scale. Adjacent grid points differ by a factor of two, so
\[
2^{-1/2}\leq\frac{\lambda}{\lambda_*}\leq2^{1/2},
\qquad
\frac{a_n}{\lambda}+b_n\lambda
\leq\frac3{\sqrt2}\sqrt{a_n b_n}.
\]
If \(\lambda_*>\lambda_+\), choose \(\lambda=\lambda_+\) and use \(b_n\lambda_+\leq\sqrt{a_n b_n}\). In both cases, and also when \(b_n=0\),
\[
\min_{\lambda\in\Lambda_n}
\left\{\frac{a_n}{\lambda}+b_n\lambda\right\}
\leq\frac3{\sqrt2}\sqrt{a_n b_n}+\frac{a_n}{\lambda_+}
\leq3\sqrt{\frac{r_nv_2a_n}{n}}
+\frac{2a_nu_n}{n}.
\]
Only this comparison value uses the unknown quantities \(a_n,b_n\); the implemented choice remains the empirical minimizer \(\widehat\lambda\).

Corollary~\ref{cor:deviation-exponential} supplies
\[
\beta(u_n)\leq C_{\mathrm e}e^{-\kappa u_n}
\left(u_n+\frac1{1-e^{-\kappa}}\right),
\qquad
n-\mathbb E[D_n]\leq
\frac{C_{\mathrm e}e^{-\kappa}}{(1-e^{-\kappa})^2}.
\]
Insert these inequalities and the grid bound into \eqref{eq:adaptive-oracle-before-grid} to obtain \eqref{eq:oracle-exponential-adaptive}.

For fixed \(c,d\) and a fixed regenerative law, \(u_n=O((\log n)^2)\), \(K_n=O(\log n)\), and \(a_n=O(d\log n)\). Also \(n/\mathbb E[D_n]=1+O(n^{-1})\), \(r_n\to1/m\), and \(v_2<\infty\). Finally,
\[
e^{-\kappa u_n}\leq e^{-\kappa(\log n)^2}
=n^{-\kappa\log n},
\]
so \(\beta(u_n)=o(n^{-a})\) for every fixed \(a>0\). The square-root term has order \(\sqrt{d\log n/n}\), the truncation and grid corrections have order \(d(\log n)^3/n\), and the renewal-age term is \(O(n^{-1})\). This proves \eqref{eq:oracle-exponential-adaptive-rate} and the stated second-order property.
\end{proof}

\subsection{Proof of Proposition~\ref{prop:moment-confidence}}

\begin{proof}[Proof of Proposition~\ref{prop:moment-confidence}]
We first control the number and lengths of the regeneration cycles. These bounds give an event on which \(D_n\geq n/2\). On that event, we estimate each term of Theorem~\ref{thm:regenerative-bernstein}.

The moment assumption and \(C\geq1\) imply
\begin{equation*}
\mathbb E_A[\Delta_1^2]\leq Cm^2,
\qquad
\beta(u_n)\leq\frac{Cm^{2+\varepsilon}}{u_n^{1+\varepsilon}}.
\end{equation*}
Since \(M_n\) is a bounded stopping time, Wald's identity \cite[Chapter~10, Section~14.2]{gut_probability_2005} gives
\[
\begin{split}
\mathbb E\left[\sum_{j=1}^{M_n}\Delta_j\right]
&=m\mathbb E[M_n],\\
\mathbb E\left[\sum_{j=1}^{M_n}\Delta_j^2\right]
&=\mathbb E_A[\Delta_1^2]\mathbb E[M_n].
\end{split}
\]
Every cycle before \(M_n\) ends by time \(n\), whereas cycle \(M_n\) ends after time \(n\). Thus
\[
\sum_{j=1}^{M_n}\Delta_j=D_n+\Delta_{M_n}\leq n+\Delta_{M_n},
\qquad
\Delta_{M_n}^2\leq\sum_{j=1}^{M_n}\Delta_j^2.
\]
Cauchy--Schwarz consequently gives
\[
\begin{split}
m\mathbb E[M_n]
&\leq n+\sqrt{\mathbb E_A[\Delta_1^2]\mathbb E[M_n]}\\
&\leq n+\frac m2\mathbb E[M_n]
+\frac{\mathbb E_A[\Delta_1^2]}{2m}.
\end{split}
\]
Rearranging and using \(n\geq48Cm/\delta\geq Cm\), we obtain
\[
\mathbb E[M_n]
\leq\frac{2n}{m}+\frac{\mathbb E_A[\Delta_1^2]}{m^2}
\leq\frac{3n}{m}.
\]
The same stopping-time identity now yields
\[
\begin{split}
\mathbb E\left[\sum_{j=1}^{M_n}\Delta_j^2\right]
&\leq3Cmn,\\
\mathbb E\left[\sum_{j=1}^{M_n}\Delta_j
\mathbf 1_{\{\Delta_j>u_n\}}\right]
&\leq\frac{3n\beta(u_n)}m.
\end{split}
\]

Apply Markov's inequality separately to the three quantities. Each resulting bound holds with probability at least \(1-\delta/4\). A union bound gives an event of probability at least \(1-3\delta/4\) on which
\begin{equation}
\begin{split}
M_n&\leq\frac{12n}{m\delta},\\
\sum_{j=1}^{M_n}\Delta_j^2&\leq\frac{12Cmn}{\delta},\\
\sum_{j=1}^{M_n}\Delta_j\mathbf 1_{\{\Delta_j>u_n\}}
&\leq\frac{12n\beta(u_n)}{m\delta}.
\end{split}
\label{eq:moment-length-event}
\end{equation}
If \(\beta(u_n)=0\), the third sum is zero almost surely. On this event, the length of cycle \(M_n\) satisfies
\[
\Delta_{M_n}^2\leq\frac{12Cmn}{\delta}\leq\frac{n^2}{4}.
\]
Since cycle \(M_n\) ends after time \(n\),
\[
D_n+\Delta_{M_n}>n,\qquad D_n\geq\frac n2.
\]

We now apply Theorem~\ref{thm:regenerative-bernstein} with \(\delta\) replaced by \(\delta/4\). Its event and the bounds in \eqref{eq:moment-length-event} hold together with probability at least \(1-\delta\). We estimate the terms on this intersection. The first term is bounded by
\[
\frac{n}{\lambda D_n}\log\frac8\delta
\leq\frac2\lambda\log\frac8\delta.
\]
For the block count and quadratic terms, use \(\log(1+x)\leq x\) and \(q(x)\leq x^2\) to obtain
\[
\begin{split}
\frac{nM_n}{\lambda D_n}
\log\left(1+\frac{\lambda\beta(u_n)}n\right)
&\leq\frac{24\beta(u_n)}{m\delta},\\
\frac n{\lambda D_n}
\sum_{j=1}^{M_n-1}\mathbf 1_{\{\Delta_j\leq u_n\}}
q\left(\frac{\lambda\Delta_j}n\right)
&\leq\frac{24C\lambda m}{\delta n}.
\end{split}
\]
Finally, we bound the correction for this additional cycle and the losses discarded by truncation:
\[
\begin{split}
-\frac n{\lambda D_n}\log\left(1-\frac{\lambda u_n}n\right)
&\leq\frac{4u_n}{n},\\
\frac1{D_n}\sum_{j=1}^{M_n-1}
\Delta_j\mathbf 1_{\{\Delta_j>u_n\}}
&\leq\frac{24\beta(u_n)}{m\delta}.
\end{split}
\]
Here the logarithmic and quadratic estimates are valid because \(\lambda u_n/n\leq1/2\). The chosen truncation level gives
\[
\frac{u_n}{n}=\left(\frac mn\right)^{(1+\varepsilon)/(2+\varepsilon)},
\qquad
\frac{\beta(u_n)}m
\leq C\left(\frac mn\right)^{(1+\varepsilon)/(2+\varepsilon)}.
\]
Adding the five estimates proves
\begin{align*}
\left|R(\theta)-\widehat R_n^{\mathrm{reg}}(\theta)\right|
\leq{}&\frac2\lambda\log\frac8\delta
+\frac{24C\lambda m}{\delta n}
\notag\\
&+\left(\frac{48C}{\delta}+4\right)
\left(\frac mn\right)^{(1+\varepsilon)/(2+\varepsilon)}.
\end{align*}
This is \eqref{eq:moment-confidence}.
\end{proof}

\subsection{Proof of Theorem~\ref{thm:moment-minimax}}

\begin{proof}[Theorem~\ref{thm:moment-minimax}]
We obtain the upper bound by choosing \(\lambda\) in Proposition~\ref{prop:moment-confidence}. For the lower bound, we use deterministic cycle lengths and Bernoulli losses that are constant within each cycle. We then compare both bounds over \(\mathfrak P(m,C,\varepsilon)\).

\medskip
\noindent\textbf{Upper bound.} Fix \(\theta\in\Theta\) and choose the deterministic value
\[
\lambda:=\min\left\{
\sqrt{\frac{\delta n}{12Cm}\log\frac8\delta},\frac n{2u_n}
\right\}.
\]
When the first expression is at most \(n/(2u_n)\), the first two terms of \eqref{eq:moment-confidence} have sum
\[
8\sqrt{\frac{3Cm}{\delta n}\log\frac8\delta}.
\]
Otherwise, \(\lambda=n/(2u_n)\), and the reciprocal term equals
\[
\frac{4u_n}{n}\log\frac8\delta,
\]
while the linear term is no larger than at the unconstrained minimizer. Consequently, this choice satisfies
\[
\frac2\lambda\log\frac8\delta+\frac{24C\lambda m}{\delta n}
\leq8\sqrt{\frac{3Cm}{\delta n}\log\frac8\delta}
+\frac{4u_n}{n}\log\frac8\delta.
\]
Since \(n\geq m\) and \((1+\varepsilon)/(2+\varepsilon)>1/2\),
\[
\frac{u_n}{n}
=\left(\frac mn\right)^{(1+\varepsilon)/(2+\varepsilon)}
\leq\sqrt{\frac mn}.
\]
Adding the remaining term of \eqref{eq:moment-confidence} therefore gives
\begin{equation}
\begin{split}
\left|R(\theta)-\widehat R_n^{\mathrm{reg}}(\theta)\right|
&\leq\left[
8\sqrt{\frac{3C}{\delta}\log\frac8\delta}
+4\log\frac8\delta+\frac{48C}{\delta}+4\right]\sqrt{\frac mn}\\
&=C_{C,\delta}\sqrt{\frac mn}.
\end{split}
\label{eq:moment-confidence-optimized}
\end{equation}
This inequality holds on the event from Proposition~\ref{prop:moment-confidence}, which has probability at least \(1-\delta\) and satisfies \(D_n\geq n/2\).

\medskip
\noindent\textbf{Lower bound.} Suppose that \(m\in\mathbb N\) and \(n\geq m\). Fix \(p\in[0,1]\) and restrict to cycles of deterministic length \(\Delta_j=m\). In cycle \(j\), let all observations share an independent Bernoulli label:
\[
Y_t=U_j\quad\text{for }\tau_j<t\leq\tau_{j+1},
\qquad U_j\sim\operatorname{Bernoulli}(p).
\]
Denote this regenerative law by \(P_p\). Its risk and length moment are
\[
R_{P_p}(\theta_0)=p,
\qquad\mathbb E_{A,P_p}[\Delta_1^{2+\varepsilon}]
=m^{2+\varepsilon}\leq Cm^{2+\varepsilon},
\]
so it belongs to \(\mathfrak P(m,C,\varepsilon)\). The observed path, including its possibly incomplete last cycle, determines exactly
\[
k_n:=\left\lceil\frac nm\right\rceil
\]
independent cycle labels. Conversely, these labels determine the path; the regeneration boundaries are deterministic. Thus estimating the long-run risk in this submodel is equivalent to estimating \(p\) from \(k_n\) independent Bernoulli observations.

Choose two possible values of \(p\):
\[
h_n:=\frac1{16\sqrt{k_n}},\qquad
p_-:=\frac12-h_n,\qquad p_+:=\frac12+h_n.
\]
Let \(P_-^{(n)}\) and \(P_+^{(n)}\) denote their laws on the observed sample and regeneration boundaries. The reduction to independent Bernoulli observations gives
\[
\begin{split}
\KL(P_-^{(n)}\Vert P_+^{(n)})
&=k_n\KL\bigl(\operatorname{Bernoulli}(p_-)
\Vert\operatorname{Bernoulli}(p_+)\bigr)\\
&=2k_nh_n\log\left(\frac{1+2h_n}{1-2h_n}\right).
\end{split}
\]
Using
\[
\log\left(\frac{1+x}{1-x}\right)\leq\frac{2x}{1-x},
\qquad 0\leq x<1,
\]
and \(h_n\leq1/16\), we obtain
\[
\KL(P_-^{(n)}\Vert P_+^{(n)})
\leq\frac{8k_nh_n^2}{1-2h_n}
\leq32k_nh_n^2=\frac18.
\]
Pinsker's inequality consequently yields
\begin{equation*}
\operatorname{TV}(P_-^{(n)},P_+^{(n)})
\leq\sqrt{\frac12\KL(P_-^{(n)}\Vert P_+^{(n)})}
\leq\frac14.
\end{equation*}

To turn this testing bound into an estimation bound, let any estimator \(\widetilde R_n\) define the test
\[
\varphi_n:=\mathbf 1_{\{\widetilde R_n\geq1/2\}}.
\]
A testing error implies an estimation error of at least \(h_n\), since
\[
\begin{split}
\{\varphi_n=1\}&\subseteq\{|\widetilde R_n-p_-|\geq h_n\},\\
\{\varphi_n=0\}&\subseteq\{|\widetilde R_n-p_+|\geq h_n\}.
\end{split}
\]
It follows that
\[
\begin{split}
\max_{\sigma\in\{-,+\}}
P_\sigma^{(n)}(|\widetilde R_n-p_\sigma|\geq h_n)
&\geq\frac12\left\{P_-^{(n)}(\varphi_n=1)
+P_+^{(n)}(\varphi_n=0)\right\}\\
&\geq\frac{1-\operatorname{TV}(P_-^{(n)},P_+^{(n)})}{2}\\
&\geq\frac38.
\end{split}
\]
For \(n\geq m\),
\[
k_n\leq\frac{2n}{m},\qquad
h_n\geq\frac1{16\sqrt2}\sqrt{\frac mn}.
\]
The two laws lie in the model class, and the estimator was arbitrary. Hence
\begin{equation}
\inf_{\widetilde R_n}
\sup_{P\in\mathfrak P(m,C,\varepsilon)}
\mathbb P_{A,P}\left(
|\widetilde R_n-R_P(\theta_0)|
\geq\frac1{16\sqrt2}\sqrt{\frac mn}
\right)
\geq\frac38.
\label{eq:minimax-lower}
\end{equation}

\medskip
\noindent\textbf{Minimax bounds.} Apply \eqref{eq:moment-confidence-optimized} with \(\theta=\theta_0\) to each law in \(\mathfrak P(m,C,\varepsilon)\). Its constants are the same for every law in this class. Setting the empirical risk to \(1/2\) when \(D_n=0\) does not affect the event where \(D_n\geq n/2\), so
\[
\sup_{P\in\mathfrak P(m,C,\varepsilon)}
\mathbb P_{A,P}\left(
|\widehat R_n^{\mathrm{reg}}(\theta_0)-R_P(\theta_0)|
>C_{C,\delta}\sqrt{\frac mn}\right)\leq\delta.
\]
This proves the upper inequality in \eqref{eq:minimax-match}. For the lower inequality, every threshold satisfying
\[
0\leq r<\frac1{16\sqrt2}\sqrt{\frac mn}
\]
also satisfies, by \eqref{eq:minimax-lower},
\[
\inf_{\widetilde R_n}\sup_{P\in\mathfrak P(m,C,\varepsilon)}
\mathbb P_{A,P}(|\widetilde R_n-R_P(\theta_0)|>r)
\geq\frac38>\delta.
\]
Such a threshold cannot satisfy the probability requirement in the definition of \(r_{n,\delta}(m,C,\varepsilon)\). Hence the lower inequality in \eqref{eq:minimax-match} follows.
\end{proof}

\subsection{Proof of Theorem~\ref{thm:regenerative-gaussian}}

\begin{proof}[Proof of Theorem~\ref{thm:regenerative-gaussian}]
We first couple the centered sums with Brownian motion. We then use the Bernstein inequality to control empirical centering, prove the variance estimator's convergence rate, and show that replacing the true variance by its estimate preserves the approximation. All limits below are almost sure unless stated otherwise. Recall that
\[
p=2+\varepsilon,
\qquad N_n=M_n-1.
\]
We use the strong law, the Marcinkiewicz--Zygmund strong law, and the law of the iterated logarithm in the forms given in \cite[Sections~6.6, 6.7 and~8.1]{gut_probability_2005}.

Observe by Tonelli's theorem that for every \(b>0\), 
\[
\begin{split}
\sum_{j=1}^{\infty}\mathbb P_A(\Delta_j>b j^{1/p})
&=\mathbb E_A\left[\sum_{j=1}^{\infty}
\mathbf 1_{\{\Delta_1^p>b^p j\}}\right]\\
&\leq b^{-p}\mathbb E_A[\Delta_1^p]<\infty.
\end{split}
\]
Then, by Borel--Cantelli lemma, applied to positive rational values of \(b\), implies
\[
\Delta_j=o(j^{1/p}).
\]
Taking a maximum over a fixed initial segment and its complement, and using \(M_n\leq n+1\), therefore yields
\begin{equation}
\max_{j\leq n+1}\Delta_j=o(n^{1/p}),
\qquad
0\leq n-D_n<\Delta_{M_n}=o(n^{1/p}).
\label{eq:gaussian-renewal-remainder}
\end{equation}
Further recall the inequalities
\[
D_n\leq n<D_n+\Delta_{M_n};
\]
which, along with the strong law of large numbers
\[
\frac{N_n}{n}\longrightarrow\frac1m,
\qquad
\frac{D_n}{n}\longrightarrow1.
\]

We next check the hypotheses of  \citet[Theorem 1]{bashtova_shashkin_heavy_2020}. Extend $S_{n,\theta}^{\text{full}}$ to continuous time by
\[
A_\theta(t):=S_{\lfloor t\rfloor,\theta}^{\mathrm{full}},
\qquad t\geq0.
\]
This step process is right-continuous and separable. Its cycle paths and lengths are i.i.d., since each is a measurable function of one regenerative block. The cycle reward has mean zero, and the absolute increment at any time within a cycle is bounded by its length:
\[
\mathbb E_A[\xi_{1,\theta}]=0,
\qquad
\sup_{0\leq t\leq\Delta_j}
\left|A_\theta(\tau_j+t)-A_\theta(\tau_j)\right|
\leq\Delta_j.
\]
This verifies both moment hypotheses of \citet[Theorem~1, p.~2]{bashtova_shashkin_heavy_2020}. The variance per unit time in that theorem is
\[
\frac{\mathbb E_A[\xi_{1,\theta}^2]}m=\sigma_\theta^2.
\]
It follows that the full-path sum admits a Brownian coupling satisfying
\[
\max_{0\leq k\leq n}
\left|S_{k,\theta}^{\mathrm{full}}-\sigma_\theta W_\theta(k)\right|
=o(n^{1/p}).
\]
We retain the joint law of the loss paths and regeneration times in this realization. Observe that removing the incomplete cycle changes the sum by at most the longest cycle, because
\[
S_{k,\theta}^{\mathrm{reg}}=S_{D_k,\theta}^{\mathrm{full}}.
\]
Consequently,
\[
\max_{0\leq k\leq n}
\left|S_{k,\theta}^{\mathrm{full}}-S_{k,\theta}^{\mathrm{reg}}\right|
\leq\max_{j\leq n+1}\Delta_j=o(n^{1/p}).
\]
The triangle inequality proves \eqref{eq:regenerative-strong-approximation} for both sums with the same Brownian motion.

\medskip
The variance estimator centers each block at the empirical risk. We use Theorem~\ref{thm:regenerative-bernstein} to control the difference between this empirical risk and \(R(\theta)\), which will then bound the change in the quadratic sum. Keep the truncation level in the statement and choose
\[
u_n=m(n/m)^{1/p},
\qquad \delta_n=(n+2)^{-2},
\qquad
\lambda_n=\min\left\{\sqrt{\frac nm\log(n+2)},\frac n{2u_n}\right\}.
\]
These choices satisfy the admissibility condition for every \(n\). Moreover,
\[
\sqrt{\frac nm\log(n+2)}\,\frac{u_n}{n}
=\sqrt{\log(n+2)}\left(\frac mn\right)^{1/2-1/p}
\longrightarrow0,
\]
so the constraint on \(\lambda_n\) is eventually inactive.

For each \(n\), apply inequality~\eqref{eq:regenerative-bernstein} in Theorem~\ref{thm:regenerative-bernstein} with \(\lambda=\lambda_n\), \(\delta=\delta_n\), and the truncation level \(u_n\) chosen above. The resulting bound holds with probability at least \(1-\delta_n\). Since \(\sum_n\delta_n<\infty\), Borel--Cantelli implies that the bound holds for all sufficiently large \(n\), almost surely. Also, \eqref{eq:gaussian-renewal-remainder} implies
\[
\frac{\max_{j\leq M_n}\Delta_j}{u_n}\longrightarrow0.
\]
Hence all completed cycles eventually satisfy \(\Delta_j\leq u_n\). The truncation indicators in \eqref{eq:regenerative-bernstein} then equal one, and its sum over longer completed cycles is zero. The first inequality in the calculation below is exactly the application of Theorem~\ref{thm:regenerative-bernstein} with the stated parameters. The second inequality uses only
\[
q(x)\leq x^2,\qquad -\log(1-x)\leq2x
\quad(0\leq x\leq1/2),
\qquad
\log(1+y)\leq y\quad(y\geq0):
\]
\begin{equation}
\begin{split}
\left|\widehat R_n^{\mathrm{reg}}(\theta)-R(\theta)\right|
\overset{\text{Theorem~\ref{thm:regenerative-bernstein}}}{\leq}{}&
\frac{n}{\lambda_nD_n}\Bigg[
\log\frac{2}{\delta_n}
+M_n\log\left(1+\frac{\lambda_n\beta(u_n)}n\right)\\
&\quad+\sum_{j=1}^{N_n}q\left(\frac{\lambda_n\Delta_j}n\right)
-\log\left(1-\frac{\lambda_nu_n}n\right)\Bigg]\\
\leq{}&\frac{n\log(2/\delta_n)}{\lambda_nD_n}
+\frac{M_n}{D_n}\beta(u_n)\\
&+\frac{\lambda_n}{nD_n}\sum_{j=1}^{N_n}\Delta_j^2
+\frac{2u_n}{D_n}.
\end{split}
\label{eq:gaussian-bernstein-centering}
\end{equation}
The strong law and the moment assumption give:
\[
\frac{M_n}{D_n}\longrightarrow\frac1m,
\qquad
\frac1{D_n}\sum_{j=1}^{N_n}\Delta_j^2
\longrightarrow\frac{\mathbb E_A[\Delta_1^2]}m
\leq C^{2/p}m,
\]
and
\[
\log(2/\delta_n)\leq3\log(n+2),
\qquad
\frac{\beta(u_n)}m\leq C\left(\frac mn\right)^{1-1/p},
\qquad
\frac{u_n}{n}=\left(\frac mn\right)^{1-1/p}.
\]
In particular, the following bounds hold eventually:
\[
D_n\geq n/2,
\qquad M_n/D_n\leq2/m,
\qquad
\frac1{D_n}\sum_{j=1}^{N_n}\Delta_j^2\leq2C^{2/p}m.
\]
Substituting these bounds into the final inequality in \eqref{eq:gaussian-bernstein-centering} gives
\begin{align}
\left|\widehat R_n^{\mathrm{reg}}(\theta)-R(\theta)\right|
&\leq(6+2C^{2/p})\sqrt{\frac{m\log(n+2)}n}
+(2C+4)\left(\frac mn\right)^{1-1/p}
\notag\\
&=O\!\left(\sqrt{\frac{\log(n+2)}n}\right).
\label{eq:gaussian-centering-rate}
\end{align}
Here the law is fixed and \(1-1/p>1/2\). We use the resulting almost-sure estimate \eqref{eq:gaussian-centering-rate} below to control the change in the quadratic sum caused by empirical centering. %The sample-size condition in Proposition~\ref{prop:moment-confidence} would fail for our summable sequence \(\delta_n\).

\medskip
First consider the quadratic sum centered at the true risk. Define
\[
Q_{n,\theta}:=\sum_{j=1}^{N_n}\xi_{j,\theta}^2,
\qquad
e_{n,\theta}:=\widehat R_n^{\mathrm{reg}}(\theta)-R(\theta).
\]
The law of the iterated logarithm for the cycle lengths, evaluated at \(N_n\sim n/m\), gives
\[
D_n-mN_n=O\!\left(\sqrt{n\log\log(n+e^e)}\right),
\qquad
\frac{N_n}{D_n}-\frac1m
=O\!\left(\sqrt{\frac{\log\log(n+e^e)}n}\right).
\]
Both bounds also hold when the length variance is zero. For \(2<p<4\), the centered variables \(\xi_{j,\theta}^2-v_\theta^2\) have finite \(p/2\)-moment, so the Marcinkiewicz--Zygmund strong law applies. For \(p\geq4\), they have finite variance, so the law of the iterated logarithm applies instead. Together, these give
\[
Q_{n,\theta}-v_\theta^2N_n=
\begin{cases}
o(n^{2/p}),&2<p<4,\\
O\!\left(\sqrt{n\log\log(n+e^e)}\right),&p\geq4.
\end{cases}
\]
The exact decomposition
\[
\frac{Q_{n,\theta}}{D_n}-\sigma_\theta^2
=\frac{Q_{n,\theta}-v_\theta^2N_n}{D_n}
+v_\theta^2\left(\frac{N_n}{D_n}-\frac1m\right)
\]
therefore yields
\[
\frac{Q_{n,\theta}}{D_n}-\sigma_\theta^2=
\begin{cases}
o(n^{2/p-1}),&2<p<4,\\
O\!\left(\sqrt{\log\log(n+e^e)/n}\right),&p\geq4.
\end{cases}
\]
In the first case, the renewal-count error is smaller than the stated rate because
\[
n^{1/2-2/p}\sqrt{\log\log(n+e^e)}\longrightarrow0.
\]

It remains to replace the true-risk centering by empirical centering. For each block,
\[
\Delta_j\widehat R_n^{\mathrm{reg}}(\theta)-L_{j,\theta}
=\xi_{j,\theta}+\Delta_je_{n,\theta}.
\]
Squaring and summing gives
\begin{equation*}
\widehat Q_{n,\theta}-Q_{n,\theta}
=2e_{n,\theta}\sum_{j=1}^{N_n}\Delta_j\xi_{j,\theta}
+e_{n,\theta}^2\sum_{j=1}^{N_n}\Delta_j^2.
\end{equation*}
The block bound and the strong law imply
\[
|\xi_{j,\theta}|\leq\Delta_j,
\qquad
\sum_{j=1}^{N_n}\Delta_j^2=O(n).
\]
We now substitute the estimate obtained from Theorem~\ref{thm:regenerative-bernstein} into this variance calculation. Since \(e_{n,\theta}\) is the empirical risk error, \eqref{eq:gaussian-centering-rate} says that, almost surely,
\[
|e_{n,\theta}|=O\!\left(\sqrt{\frac{\log(n+2)}n}\right),
\qquad
e_{n,\theta}^2=O\!\left(\frac{\log(n+2)}n\right).
\]
Inserting these two bounds and \(\sum_{j=1}^{N_n}\Delta_j^2=O(n)\) into the preceding expansion gives
\[
\begin{split}
|\widehat Q_{n,\theta}-Q_{n,\theta}|
&\leq(2|e_{n,\theta}|+e_{n,\theta}^2)
\sum_{j=1}^{N_n}\Delta_j^2\\
&\overset{\eqref{eq:gaussian-centering-rate}}{=}
O\!\left(n\sqrt{\frac{\log(n+2)}n}+\log(n+2)\right)\\
&=O\!\left(\sqrt{n\log(n+2)}\right).
\end{split}
\]
Dividing by \(D_n\sim n\) proves \eqref{eq:regenerative-variance-rate}. For \(2<p<4\), this extra error is negligible because
\[
n^{1/2-2/p}\sqrt{\log(n+2)}\longrightarrow0.
\]
In particular,
\[
\widehat\sigma_{n,\theta}\longrightarrow\sigma_\theta>0.
\]

\medskip
The normalization error is the variance-estimation error multiplied by the maximum Brownian fluctuation. For \(T>0\) and \(x>0\), the exponential inequality for Brownian motion \cite[Chapter~II, Proposition~1.8, p.~55]{revuz1999continuous}, applied to \(W_\theta\) and \(-W_\theta\), gives
\[
\mathbb P\left(\sup_{0\leq t\leq T}|W_\theta(t)|>x\right)
\leq4e^{-x^2/(2T)}.
\]
Apply this at dyadic times with
\[
T=2^r,
\qquad x=3\sqrt{2^r\log r},
\qquad r\in\{2,3,\ldots\}.
\]
The resulting probabilities are summable. Borel--Cantelli and comparison with the nearest dyadic time imply
\[
\frac{\max_{0\leq k\leq n}|W_\theta(k)|}{\sqrt n}
=O\!\left(\sqrt{\log\log(n+e^e)}\right).
\]
For either centered sum, the triangle inequality separates the coupling and variance-estimation errors:
\begin{align*}
\max_{0\leq k\leq n}
\left|\frac{S_{k,\theta}^{H}}{\widehat\sigma_{n,\theta}\sqrt n}
-\frac{W_\theta(k)}{\sqrt n}\right| &\quad\leq
\frac{\max_{0\leq k\leq n}|S_{k,\theta}^{H}-\sigma_\theta W_\theta(k)|}
{\widehat\sigma_{n,\theta}\sqrt n}\\
&\qquad\qquad +\left|\frac{\sigma_\theta}{\widehat\sigma_{n,\theta}}-1\right|
\frac{\max_{0\leq k\leq n}|W_\theta(k)|}{\sqrt n}.
\end{align*}
The strong approximation \eqref{eq:regenerative-strong-approximation} bounds the first term by \(o(n^{1/p-1/2})\). Positivity of the limiting variance and \eqref{eq:regenerative-variance-rate} bound the second term by
\[
\begin{cases}
o(n^{2/p-1})\,O\!\left(\sqrt{\log\log(n+e^e)}\right),&2<p<4,\\
O\!\left(n^{-1/2}\sqrt{\log(n+2)\log\log(n+e^e)}\right),&p\geq4.
\end{cases}
\]
In both cases, division by \(n^{1/p-1/2}\) makes the bound tend to zero. This proves \eqref{eq:regenerative-studentized-approximation}.

Next replace the sample-size normalization by the observed quadratic sum. The relation
\[
\sqrt{\widehat Q_{n,\theta}}
=\widehat\sigma_{n,\theta}\sqrt{D_n}
\]
changes the normalized sum by the factor \(\sqrt{n/D_n}\). By \eqref{eq:gaussian-renewal-remainder},
\[
\left|\sqrt{\frac n{D_n}}-1\right|
=\frac{n-D_n}{\sqrt{D_n}(\sqrt n+\sqrt{D_n})}
=o(n^{1/p-1}).
\]
The additional error is therefore bounded by
\[
\begin{split}
&\sqrt{\frac n{D_n}}\,o(n^{1/p-1/2})
+\left|\sqrt{\frac n{D_n}}-1\right|
\frac{\max_{0\leq k\leq n}|W_\theta(k)|}{\sqrt n}\\
&\hspace{5em}=o(n^{1/p-1/2}).
\end{split}
\]

Finally, pass from integer times to all \(t\in[0,1]\). The same exponential inequality on each unit interval and a union bound yield
\[
\begin{split}
&\mathbb P\left(
\max_{0\leq k\leq n}\sup_{0\leq u\leq1}
|W_\theta(k+u)-W_\theta(k)|>3\sqrt{\log(n+2)}
\right)\\
&\qquad\leq4(n+1)(n+2)^{-9/2}.
\end{split}
\]
Using Borel--Cantelli lemma then gives the uniform interpolation bound
\[
\sup_{0\leq t\leq1}
\frac{|W_\theta(\lfloor nt\rfloor)-W_\theta(nt)|}{\sqrt n}
=O\!\left(\sqrt{\frac{\log(n+2)}n}\right)
=o(n^{1/p-1/2}).
\]
This proves \eqref{eq:regenerative-functional-approximation}. At \(t=1\), the identities
\[
S_{n,\theta}^{\mathrm{reg}}
=D_n\{R(\theta)-\widehat R_n^{\mathrm{reg}}(\theta)\},
\qquad
\frac{W_\theta(n)}{\sqrt n}\sim\mathcal N(0,1)
\]
give \eqref{eq:regenerative-studentized-clt} by Slutsky's theorem. Brownian scaling makes the comparison process
\[
t\longmapsto\frac{W_\theta(nt)}{\sqrt n}
\]
a standard Brownian motion in distribution for each \(n\). The vanishing uniform approximation error therefore also proves the stated convergence in \(D[0,1]\).
\end{proof}

\bibliographystyle{plainnat}
\bibliography{references}

\end{document}